\documentclass{article}

\usepackage{microtype}
\usepackage{graphicx}
\usepackage{subcaption}
\usepackage{booktabs} % for professional tables

\usepackage{hyperref}

\usepackage[preprint]{icml2026}

\usepackage{amsmath}
\usepackage{amssymb}
\usepackage{mathtools}
\usepackage{amsthm}

\usepackage{xcolor} 
\usepackage{tcolorbox}
\definecolor{myblue}{rgb}{0.854,0.910,0.98}
\definecolor{myred}{rgb}{1,0.8,0.8}

\usepackage[table,xcdraw]{xcolor}
\usepackage{multirow}
\usepackage{placeins}

\usepackage{mdframed}
\usepackage[capitalize,noabbrev]{cleveref}

\theoremstyle{plain}
\newtheorem{theorem}{Theorem}[section]

\theoremstyle{definition}

\newtheorem{assumption}[theorem]{Assumption}
\theoremstyle{remark}

\usepackage[textsize=tiny]{todonotes}

\icmltitlerunning{Judge, Then Drive: A Critic-Centric Vision Language Action Framework for Autonomous Driving}

\begin{document}

\twocolumn[
  \icmltitle{\raisebox{-1.0ex}{\includegraphics[width=0.06\textwidth]{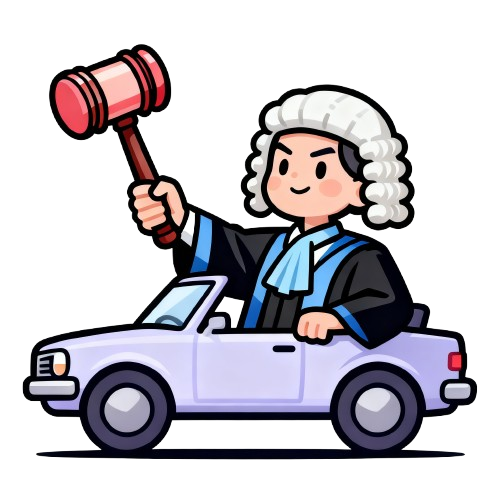}} Judge, Then Drive: A Critic-Centric Vision Language Action Framework for Autonomous Driving}

  % It is OKAY to include author information, even for blind submissions: the
  % style file will automatically remove it for you unless you've provided
  % the [accepted] option to the icml2026 package.

  % List of affiliations: The first argument should be a (short) identifier you
  % will use later to specify author affiliations Academic affiliations
  % should list Department, University, City, Region, Country Industry
  % affiliations should list Company, City, Region, Country

  % You can specify symbols, otherwise they are numbered in order. Ideally, you
  % should not use this facility. Affiliations will be numbered in order of
  % appearance and this is the preferred way.
  \icmlsetsymbol{equal}{*}

  \begin{icmlauthorlist}
    \icmlauthor{Lijin Yang}{equal,yyy}
    \icmlauthor{Jianing Huang}{equal,yyy}
    \icmlauthor{Zhongzhan Huang}{equal,yyy}
    \icmlauthor{Shu Liu}{yyy}
    \icmlauthor{Hao Yang}{yyy}
  \end{icmlauthorlist}

  \icmlaffiliation{yyy}{Bosch Research}

  \icmlcorrespondingauthor{Lijin Yang}{yli6sgh@bosch.com}
  % \icmlcorrespondingauthor{Firstname2 Lastname2}{first2.last2@www.uk}

  % You may provide any keywords that you find helpful for describing your
  % paper; these are used to populate the "keywords" metadata in the PDF but
  % will not be shown in the document
  \icmlkeywords{Machine Learning, ICML}

  \vskip 0.3in
]

% this must go after the closing bracket ] following \twocolumn[ ...

% This command actually creates the footnote in the first column listing the
% affiliations and the copyright notice. The command takes one argument, which
% is text to display at the start of the footnote. The \icmlEqualContribution
% command is standard text for equal contribution. Remove it (just {}) if you
% do not need this facility.

% Use ONE of the following lines. DO NOT remove the command.
% If you have no special notice, KEEP empty braces:
\printAffiliationsAndNotice{\icmlEqualContribution}  % no special notice (required even if empty)
% Or, if applicable, use the standard equal contribution text:
% \printAffiliationsAndNotice{\icmlEqualContribution}

\begin{abstract}
Recent advances in vision language action (VLA) models have shown remarkable potential for autonomous driving by directly mapping multimodal inputs to control signals. However, previous VLA-based methods have not explicitly exploited the critic capability of VLAs to refine driving decisions, even though such capability has been well demonstrated in other LLM-based domains, thereby limiting their performance in complex closed-loop scenarios. In this work, we present a theoretically inspired two-stage framework, CriticVLA, which extends the role of VLAs from acting to judging. CriticVLA first generates a rough trajectory and then refines it through multimodal evaluation and single-step optimization guided by a VLA-based critic, yielding higher-quality driving behaviors. To support this process, we construct a large-scale synthetic dataset of 12.9 million annotated trajectories covering diverse driving scenarios, which enhances the critic’s reasoning and refinement abilities. 
Extensive closed-loop experiments on the Bench2Drive benchmark show that CriticVLA significantly surpasses state-of-the-art baselines, achieving a 73.33\% total success rate and delivering about 30\% improvement in challenging scenarios. 
\end{abstract}

\section{Introduction}
\label{sec:intro}

\begin{figure}[t]
  \centering
  % \fbox{\rule{0pt}{2in} \rule{0.9\linewidth}{0pt}}
   \includegraphics[width=1.0\linewidth]{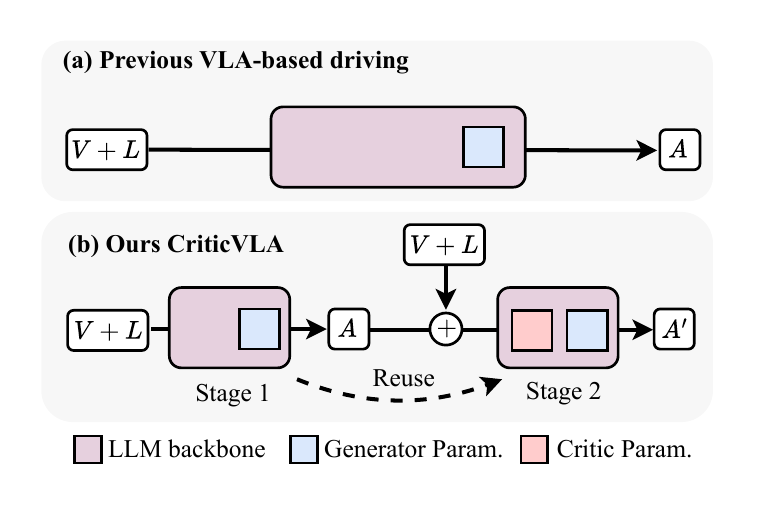}
\vspace{-0.6cm}
   \caption{Conceptual diagram of the VLA-based autonomous driving model (ADM). (a) The common paradigm of existing VLA-based ADM. (b) Our CriticVLA: in Stage-1 a VLA-based ADM generates a rough trajectory $A$ without language output, in Stage-2 the same VLA is reused to generate judgment and refinement to obtain final action $A'$. %CriticVLA explores the critic ability of VLA to refine rough trajectory generated by ADM.
   }
   \label{fig:conc}
   \vspace{-0.3cm}
\end{figure}
% \vspace{-0.3cm}
Recent advances in vision language action (VLA) models have shown remarkable potential for autonomous driving by directly mapping multimodal sensory inputs and linguistic instructions to driving actions \cite{jiang2025survey,ma2024surveyvla, sapkota2025vla, zhou2024vision,jia2024bench2drive, neurips2024drivingdojo,nuscenes2019, wilson2023argoverse2, waymo2024open,chi2024covla,fu2025orion,zhang2025safeauto,jiang2025diffvla}. Leveraging large-scale pretraining and unified multimodal representations, VLAs promise to integrate perception, reasoning, and control within a single framework, paving the way toward more interpretable and instruction-following Autonomous Driving Models (ADMs).

However, most existing VLA-based approaches still treat the model purely as an action generator that maps multimodal inputs directly to control signals. This generation-only paradigm overlooks the  \emph{evaluative} capability of the VLA’s LLM backbone, namely its ability to reason about dynamic scenes and provide constructive feedback, a capability that has been extensively demonstrated in other domains~\cite{zheng2023judging,zhong2024moextend,huang2025minilongbench,wang2023chatgpt,chiang2023closer,wang2025llava}. In safety-critical situations such as highway merging, unprotected turns, or negotiating with aggressive vehicles, the lack of an explicit critic limits the model’s capacity to detect potential risks and refine its own actions, thereby hindering robust closed-loop driving.

To address this gap, we propose \textbf{CriticVLA}, a theoretically inspired two-stage framework based on a \textit{“judge, then drive’’} philosophy. Instead of directly outputting control signals (Figure~\ref{fig:conc}a), CriticVLA first produces an initial rough trajectory (Figure~\ref{fig:conc}b). It then deploys the same VLA model, now as a multimodal critic, to systematically evaluate this plan given the current scene and generate refinements. Concretely, the critic generates structured natural-language risk analyses and action suggestions, which are consumed by a refinement head to produce high-quality final action. %We further provide a formal analysis in Section~\ref{sec:method}, showing that, under mild Lipschitz assumptions, this critic-guided refinement guarantees a multiplicative improvement of action quality in one step and admits monotonic improvement and convergence over multiple iterations.

We further construct \textbf{CriticDrive}, a large-scale synthetic dataset comprising 12.9 million annotated trajectories under diverse driving conditions to enhance the critic’s evaluative and refinement abilities. CriticDrive aggregates failed first-stage trajectories and perturbed ground-truth trajectories, providing structured annotations that include six key risk types, natural-language critiques, and refined trajectories. This design explicitly aligns perception, language, and action modalities and equips the VLA with strong critic capabilities beyond pure action generation.

In Section~\ref{sec:exp}, we conduct extensive closed-loop experiments on the challenging Bench2Drive~\cite{jia2024bench2drive} benchmark. CriticVLA significantly outperforms existing state-of-the-art ADMs, \textbf{achieving a 73.33\% total success rate and delivering about 30\% improvement in challenging scenarios}, while maintaining competitive driving efficiency. Multi-ability analysis further shows consistent gains in complex capabilities such as overtaking, merging, and handling non-signalized intersections, highlighting our \textit{“judge, then drive''} paradigm in redefining how VLA models reason about and improve driving behavior.

In summary, our main contributions are threefold: 
\begin{itemize}
    \item We formalize the importance of the critic role in VLA-based autonomous driving and derive a theoretically inspired two-stage paradigm, termed CriticVLA;
    \item We introduce CriticDrive, a large-scale synthetic dataset with 12.9 million annotated trajectories designed to enhance the critic model’s reasoning and refinement abilities;
    \item We demonstrate the superior closed-loop performance and applicability of CriticVLA on Bench2Drive, showing results consistent with our theoretical findings.
\end{itemize}

\section{Related Work}
\label{sec:related}

\noindent\textbf{VLA in Autonomous Driving.} VLA models unify perception, language-conditioned reasoning, and control in a single model, improving decision making under ambiguous observations and diverse traffic contexts \cite{ma2024surveyvla,sapkota2025vla}. Enabled by large-scale multimodal pretraining \cite{oquab2023dinov2,liu2025bridgedrive}, they show strong generalization across benchmarks such as Bench2Drive, nuScenes, Argoverse2, and Waymo \cite{nuscenes2019,wilson2023argoverse2,waymo2024open,jia2024bench2drive}. Recent progress further incorporates instruction grounding \cite{li2023driveLM}, instruction-to-trajectory computation \cite{yuan2024rag,fu2025orion}, interpretable reasoning \cite{wang2024driveCoT,cui2025chain}, and diffusion or hybrid control \cite{jiang2025diffvla}, with growing emphasis on safety and instruction awareness \cite{kim2024openvla,zhang2025safeauto}. In this work, we build on this line by using the same VLA as an explicit critic to refine driving signals.

\noindent\textbf{The Critic ability of LLM.} Many traditional automatic metrics, such as BLEU~\cite{papineni2002bleu}, ROUGE~\cite{lin2004rouge}, BERTScore~\cite{zhangbertscore}, and BARTScore~\cite{yuan2021bartscore}, are limited for open-ended evaluation and can miss nuanced qualities like helpfulness or fairness~\cite{sun2022bertscore,zhu2024starling}. With stronger LLMs~\cite{achiam2023gpt,jaech2024openai,zhong2024let,huang2025causality,huang2025routereval,zhong2022cem,liu2025associam}, the ``LLM as a judge'' paradigm has become popular \cite{zheng2023judging,wang2023chatgpt,chiang2023closer,wang2025llava,lee2024prometheus,chen2024mllm,pan2024human}. Prior work shows that, with structured prompts or pairwise comparisons, LLMs can provide fine-grained scoring, ranking, and selection beyond pure generation quality \cite{li2023prd,bai2024benchmarking}, and they are increasingly used in alignment, reasoning, and agentic decision making \cite{lee2023rlaif,liang2023encouraging,sun2024salmon,yang2023auto}. Considering the strong critic capability of LLMs, in this paper we leverage this property to guide action refinement.

\begin{figure*}[!ht]
    \centering
    \includegraphics[width=0.99\linewidth]{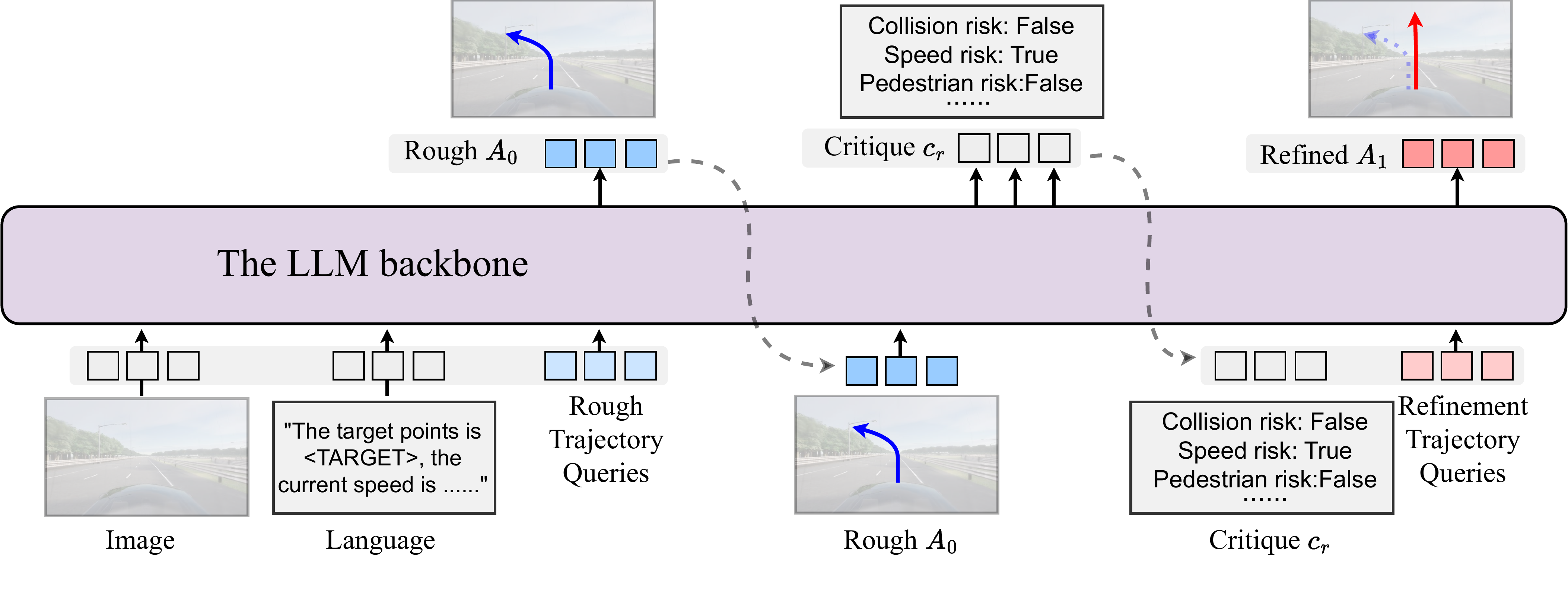}
    \caption{Overview of the proposed CriticVLA framework for autonomous driving. Stage-1: we encode the image, language instruction, and rough trajectory queries as input to an LLM to get a rough action without language output. Stage-2: the rough action is then additionally encoded and processed by LLM to generate language output as critics. Refined action is then generated from the critics and refinement trajectory queries. The LLM backbone is shared by Stage-1 and 2.}
    \label{fig:framework}
\end{figure*}

\section{Methodology}
\label{sec:method}

\subsection{Preliminary}
\label{sec:preli}

We first introduce the basic notations and mild assumptions used throughout this paper.

\paragraph{Notation.}
For VLA model, Let $\mathcal{V}$ denote the visual input space, representing all possible visual observations, and $\mathcal{L}$ the language input space, representing all possible language instructions. 
The action space $\mathcal{A}$ is defined as a metric space of trajectories equipped with a distance function $d: \mathcal{A} \times \mathcal{A} \to \mathbb{R}$ that measures the difference between actions.

We define a value function $Q: \mathcal{V} \times \mathcal{L} \times \mathcal{A} \to \mathbb{R}$, where $Q(V,L,A)$ indicates the quality of action $A$ under visual input $V$ and language input $L$. 
In autonomous driving scenarios, $Q$ typically corresponds to driving scores or success rates, and $d(\cdot,\cdot)$ can be instantiated as an Euclidean norm over trajectories, while $A$ represents trajectory signals such as speed and steering.
For simplicity, we write $Q(A) = Q(V,L,A)$ when $(V,L)$ are fixed. 
The optimal value is defined as
\begin{equation}
Q^* \equiv Q^*(V,L) = \sup_{A \in \mathcal{A}} Q(V,L,A).
\label{eq:qstar}
\end{equation}
% and the value gap between an action and the optimal action is given by
% $\Delta(V,L,A) = Q^*(V,L) - Q(V,L,A)$
% or simply $\Delta(A)$ for fixed $(V,L)$.
A VLA model is represented by a function $f: \mathcal{V} \times \mathcal{L} \to \mathcal{A}$ that generates an initial action $A_0 = f(V,L)$. 
The critic model is defined as $c: \mathcal{V} \times \mathcal{L} \times \mathcal{A} \to \mathcal{A}$, which takes $(V,L,A)$ as input. 
For fixed $(V,L)$, we define $C(A) = c(V,L,A)$.
The training set for the critic is denoted by $S \subset \mathcal{V} \times \mathcal{L} \times \mathcal{A}$, and the corresponding set of training actions is
\begin{equation}
T = \{ A \in \mathcal{A} : \exists (V,L) \text{ such that } (V,L,A) \in S \}.
\end{equation}
For any action $a \in \mathcal{A}$, its distance to the training set $T$ can be defined as
\begin{equation}
\rho(a) = \min_{A \in T} d(a, A).
\label{eq:rho}
\end{equation}

\begin{assumption}\label{ass:1}
(Critic Improvement)
The critic is assumed to have a guaranteed improvement ratio $\beta \in (0,1)$ on the training set, i.e., for all $A \in $ training set $T$,
\begin{equation}
Q(C(A)) - Q(A) \ge \beta \, [Q^*(V,L) - Q(A)].
\end{equation}
This means the critic can improve the value of any action by at least a fraction $\beta$ of its gap to the optimal value. A larger value of $\beta$ indicates a better ability of the critic.
\end{assumption}

\begin{assumption}\label{ass:2}
(Lipschitz Continuity)
There exists a constant $L_Q > 0$ such that for all $A^{(1)}, A^{(2)} \in \mathcal{A}$,
\begin{equation}
|Q(A^{(1)}) - Q(A^{(2)})| \le L_Q \, d(A^{(1)}, A^{(2)}),
\end{equation}
ensuring that small changes in actions lead to bounded changes in value. For critic function, there exists a constant $L_C > 0$ such that for all $A^{(1)}, A^{(2)} \in \mathcal{A}$,
\begin{equation}
d(C(A^{(1)}), C(A^{(2)})) \le L_C \, d(A^{(1)}, A^{(2)}),
\end{equation}
which indicates that the critic’s refinement of actions is smooth and stable.
\end{assumption}

In practice, Assumptions \ref{ass:1} and \ref{ass:2} are mild. First, Assumption \ref{ass:1} is stated only on the critic training action set $T$, and thus can be reasonably approached as long as the critic is sufficiently trained on $T$. The parameter $\beta$ reflects the critic’s improvement capability: any $\beta>0$ implies a guaranteed (possibly small) positive refinement effect relative to the initial action, and our method does not rely on near-perfect critics (i.e., $\beta\to 1$). 

Empirically, for our method, we observe an average $\beta$ estimate of roughly $\beta\approx 0.1$, suggesting that even a moderate critic can yield measurable gains.
Moreover, in ADM scenarios, the Assumption \ref{ass:2} is readily satisfied in practice, and see Appendix for more empirical evidences of these assumptions.

% \subsection{Theoretical Analysis}

\subsection{Our CriticVLA}
\label{sec:criticvla}

In this section, using the definitions and assumptions from Section \ref{sec:preli}, we derive Theorem \ref{theo:1}, which is a performance analysis of action refinement. Inspired by the Theorem \ref{theo:1}, we propose our critic-centric VLA framework for autonomous driving with two stages, called CriticVLA, as shown in Fig.~\ref{fig:conc} (b) and Fig.~\ref{fig:framework}. 

\begin{theorem}\label{theo:1}
For fixed visual input $V$ and language input $L$, let $A_0$ denote the initial action. Suppose the critic model satisfies Assumption \ref{ass:1} (critic improvement on the training set) and Assumption \ref{ass:2} (Lipschitz continuity). Let $\rho_0 = \rho(A_0)$ denote the distance from $A_0$ to the training set $T$ like Eq.~(\ref{eq:rho}). Then, the value of the improved action $A_1 = C(A_0)$ is lower bounded by
\begin{equation}
    Q(A_1) \ge \beta Q^* +  \underbrace{\colorbox[rgb]{0.854,0.910,0.988}{$(1 - \beta) \, Q(A_0)$}}_{\text{For Stage-1 in Sec.~\ref{sec:stage1}}} - \underbrace{\colorbox[rgb]{1,0.8,0.8}{$\rho_0 L_Q (1 - \beta + L_C)$}}_{\text{For Stage-2 in Sec.~\ref{sec:stage2}}},
    \label{eq:Qa1}
\end{equation}
where $Q^*$ is the upper bound value for given $V,L$ and the corresponding optimal action as defined in Eq.~(\ref{eq:qstar}).
\end{theorem}
\begin{proof}
	(See Appendix for details).\qedhere
\end{proof}
The empirical verification of Theorem \ref{theo:1} will be shown in Section \ref{sec:ana}. Actually, the Theorem \ref{theo:1} reveals an important insight: to obtain a satisfying $A_1$ with large $Q(A_1)$, the following two conditions are indispensable.

\noindent\textbf{(1) The critic model must generalize well.} Since $(1 - \beta) \in (0,1)$ and $Q(A_0) > 0$, obtaining a large $Q(A_1)$ requires a small $\rho_0$, i.e., $A_0$ should not lie far from the training action set $T$. Equivalently, $T$ must cover a broad range of rough trajectories.

\noindent\textbf{(2) The rough trajectory $A_0$ should not be too bad.} Eq.~(\ref{eq:Qa1}) shows that $Q(A_0)$ directly contributes to the lower-bound estimate of $Q(A_1)$. Therefore, if $A_0$ is severely suboptimal, even a well-generalized critic with small $\rho_0$ may still be unable to yield a high $Q(A_1)$.

Therefore, inspired by these two conditions, our framework is designed with two corresponding stages. On one hand, we can leverage various existing advanced technologies, including models and data, to obtain an good initial ADM for Stage-1, as detailed in Sections~\ref{sec:stage1}; on the other hand, we can reuse the Stage-1 prediction from the initial ADM and further synthesize large scale virtual driving data to enhance refinement ability of the critic model, which is elaborated in Sections~\ref{sec:stage2}.

\subsubsection{Stage-1: Rough Trajectory Generation }
\label{sec:stage1}

The first stage aims to generate a rough trajectory $A_0$ that provides a rough yet reasonable initialization for subsequent refinement. We build an initial VLA $f_{\text{gen}}$ by equipping an InternVL2-1B~\cite{chen2024internvl} backbone with LoRA~\cite{hu2022lora} and DETR~\cite{carion2020end} modules. Following established post-training pipelines~\cite{renz2025simlingo,liu2023visual,liu2024improved}, we fine-tune this VLA on a multi-scenario base dataset $D$ with over 2.02 million frames collected from CARLA~\cite{dosovitskiy2017carla}, DriveLM~\cite{sima2024drivelm}, and additional high-quality safe-driving data~\cite{renz2025simlingo}. For efficiency and simplicity, the model produces no textual output in this stage. Instead, the learnable DETR~\cite{carion2020end} queries are directly mapped to driving actions, yielding
\begin{equation}
    A_0 = f_{\text{gen}}(I;   \text{LoRA}_{\setlength{\fboxsep}{0pt}\fbox{\textcolor{myblue}{\rule{0.2cm}{0.2cm}}}}, \text{DETR}_{\setlength{\fboxsep}{0pt}\fbox{\textcolor{myblue}{\rule{0.2cm}{0.2cm}}}}),
    \label{eq:roughgen}
\end{equation}
where $I = (V,L)$ denote the visual and language inputs, and  $\setlength{\fboxsep}{0pt}\fbox{\textcolor{myblue}{\rule{0.25cm}{0.25cm}}} = (\text{LoRA}_{\setlength{\fboxsep}{0pt}\fbox{\textcolor{myblue}{\rule{0.2cm}{0.2cm}}}}, \text{DETR}_{\setlength{\fboxsep}{0pt}\fbox{\textcolor{myblue}{\rule{0.2cm}{0.2cm}}}})$  denote the learnable generator parameters from LoRA and DETR, respectively, as shown in Fig.~\ref{fig:conc}(b).   
The implementation details of Stage-1, including dataset organization, instruction design, and training configurations, are provided in the Appendix.
\begin{table*}[]
\centering
\caption{\textbf{Results on Bench2Drive Closed-loop Benckmark.} C/L refers to camera/LiDAR, * denote using expert feature distillation. DS represents Driving Score, SR is short for Success Rate. We conduct three independent trials with different seeds and report the mean and variance of the results.}
\label{tab:sota}
\scalebox{0.99}{
% \begin{tabular}{>{\kern-\tabcolsep}*{8}{>$c<$}<{\kern-\tabcolsep}}
\resizebox{\textwidth}{!}{
\begin{tabular}{@{}lcccccc@{}}
\toprule
\multirow{2}{*}{\textbf{Method}} & \multirow{2}{*}{\textbf{Modality}} & \multirow{2}{*}{\textbf{Venue}} & \multicolumn{4}{c}{\textbf{Closed-loop Metric}}  \cr
\cmidrule(lr){4-7}

& & & DS $\uparrow$ & SR(\%) $\uparrow$ & Efficiency $\uparrow$ & Comfortness $\uparrow$  \cr 
        \midrule
TCP*~\cite{TCP} & C & {NeurIPS' 22} & 40.70 & 15.00 & 54.26 & 47.80  \cr
TCP-traj*~ & C & {NeurIPS' 22} & 59.90 & 30.00 & 76.54 & 18.08  \cr
UniAD-Base~\cite{hu2023uniad} & C & {CVPR' 23} & 45.81 & 16.36 & 129.21 & 43.58  \cr
ThinkTwice*~\cite{jia2023think} & C & {CVPR' 23} & 62.44 & 31.23 & 69.33 & 16.22  \cr
VAD~\cite{jiang2023vad} & C & {ICCV' 23} & 42.35 & 15.00 & 157.94 & 46.01  \cr
DriveAdaptor*~\cite{jia2023driveadapter} & C+L & {ICCV' 23} & 64.22 & 33.08 & 70.22 & 16.01  \cr
GenAD~\cite{zheng2024genad} & C & {ECCV' 24} & 44.81 & 15.90 & - & -  \cr
DriveTrans~\cite{jia2025drivetransformer} & C & {ICLR' 25} & 63.46 & 35.01 & 100.64 & 20.78  \cr
TransFuser++~\cite{zimmerlin2024hidden} & C+L & {arXiv' 24} & 84.21 & 67.27 & - & -  \cr
MomAD~\cite{MomAD} & C & {CVPR' 25} & 44.54 & 16.71 & 170.21 & \textbf{48.63}  \cr
ORION~\cite{fu2025orion} & C & {ICCV' 25} & 77.74 & 54.62 & 151.48 & 17.38  \cr
Simlingo~\cite{renz2025simlingo}  & C & {CVPR' 25} & 85.07 & 67.27 & 259.23 & 33.67  \cr
\midrule
\cellcolor[HTML]{F2F2F2} CriticVLA~(\textit{Ours}) & \cellcolor[HTML]{F2F2F2} C & \cellcolor[HTML]{F2F2F2} - & \cellcolor[HTML]{F2F2F2} \textbf{88.02\small{$\pm$0.17}} & \cellcolor[HTML]{F2F2F2} \textbf{73.33\small{$\pm$0.27}} & \cellcolor[HTML]{F2F2F2} \textbf{269.24\small{$\pm$2.59}} & \cellcolor[HTML]{F2F2F2} 36.45\small{$\pm$0.50}  \cr
\bottomrule
\end{tabular}
}
}
\end{table*}

\subsubsection{Stage-2: Critic and refinement by VLA }
\label{sec:stage2}
Based on the rough $A_0$ generated in Stage-1, we consider further refining it using a critique to produce a better action. As shown in Eq.~(\ref{eq:critic-refine}), we first require the model to generate a high-quality and interpretable critique $c_r$ for $A_0$. The template of $c_r$ is in Fig.~\ref{fig:template}. Using this $c_r$ as an refinement instruction, we further employ DETR~\cite{carion2020end} with learnable queries, similar to  Eq.~(\ref{eq:roughgen}), to generate the final refined action $A_1$.  

\begin{equation}
\left\{
\begin{aligned}
c_r &= f_{\text{gen}}\!\left(A_0,I;\, \text{LoRA}_{%
\setlength{\fboxsep}{0pt}\fbox{\textcolor{myblue}{\rule{0.2cm}{0.2cm}}}%
\setlength{\fboxsep}{0pt}\fbox{\textcolor{myred}{\rule{0.2cm}{0.2cm}}}}\right),\\[6pt]
A_1 &= f_{\text{gen}}\!\left(c_r, A_0, I;\, \text{LoRA}_{%
\setlength{\fboxsep}{0pt}\fbox{\textcolor{myblue}{\rule{0.2cm}{0.2cm}}}%
\setlength{\fboxsep}{0pt}\fbox{\textcolor{myred}{\rule{0.2cm}{0.2cm}}}},\,
\text{DETR}_{\setlength{\fboxsep}{0pt}\fbox{\textcolor{myred}{\rule{0.2cm}{0.2cm}}}}\right),
\end{aligned}
\right.
\label{eq:critic-refine}
\end{equation}
where the additional $\setlength{\fboxsep}{0pt}\fbox{\textcolor{myred}{\rule{0.25cm}{0.25cm}}}$  denote the learnable parameters in Stage-2, as shown in Fig.~\ref{fig:conc} (b). Specifically, based on the analysis in Theorem \ref{theo:1}, in Stage-2, we should enhance critic's refinement capability as much as possible to reduce $\rho_0$, thereby achieving better refinement of $A_0$. 

In practice, we reuse the VLA model fine-tuned in Stage-1 as the starting point for Stage-2 training. For critique $c_r$ generation, an additional set of LoRA parameters is introduced. These parameters are combined with the LoRA parameters $\text{LoRA}_{\setlength{\fboxsep}{0pt}\fbox{\textcolor{myblue}{\rule{0.2cm}{0.2cm}}} }$ in Eq.~(\ref{eq:roughgen}) to form $\text{LoRA}_{\setlength{\fboxsep}{0pt}\fbox{\textcolor{myblue}{\rule{0.2cm}{0.2cm}}}\setlength{\fboxsep}{0pt}\fbox{\textcolor{myred}{\rule{0.2cm}{0.2cm}}}}$. Meanwhile, for generation of $A_1$, additional learnable queries $\text{DETR}_{\setlength{\fboxsep}{0pt}\fbox{\textcolor{myred}{\rule{0.2cm}{0.2cm}}}}$ are introduced.

\subsection{CriticDrive Dataset}
\label{sec:criticdrive}

\begin{figure}%[t]
    \centering
    \includegraphics[width=0.99\linewidth]{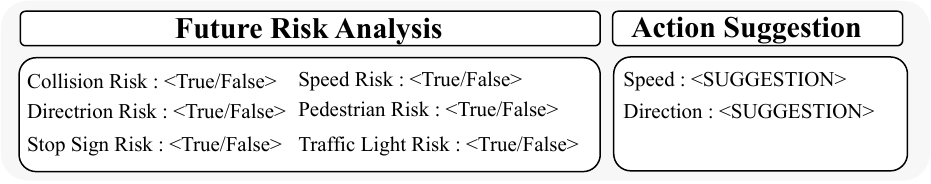}
    % \vspace{-0.3cm}
    \caption{The template of critique $c_r$. }
    \label{fig:template}
\end{figure}

Theorem~\ref{theo:1} states that critic-based refinement requires a diverse training set $T$ such that any rough trajectory from Stage-1 lies within a small distance $\rho_0$ from $T$ and the critic exhibits positive improvement ratio $\beta$ on actions in $T$. CriticDrive is explicitly constructed to meet these data requirements through two complementary subsets. 
We first construct a \textbf{Model-Generated Subset} (MGS) by collecting $2{,}023{,}499$ rough trajectories $A_0$ produced by the Stage-1 model, enabling the critic to achieve a stable improvement ratio $\beta$ over the Stage-1 model. In addition, an \textbf{Extra Perturbation-Augmented Subset} (EPAS) is constructed from $1{,}058{,}349$ ground-truth trajectories by appling diverse perturbations (e.g., longitudinal speed scaling, forced lateral offsets, explicit collision synthesis), yielding $10{,}895{,}598$ additional high-risk rough trajectories. This substantially enlarges the support of the action set $T$ and effectively reduces $\rho_0$ in practice. 

Each rough trajectory is evaluated against the expert route and scene context to automatically derive structured annotations $c_r$ using a four-step risk assessment pipeline based on deviation from ground-truth and safety consideration (details in Appendix): (1) lateral risk (angular deviations, route offset), (2) longitudinal risk (speed limit compliance, speed deviation, intent classification), (3) collision risk (3D oriented bounding boxes and separating-axis theorem tests against nearby agents and static obstacles), and (4) environmental compliance (complexity, traffic rules, right-of-way, pedestrians, adverse conditions). This pipeline outputs the True/False risk analysis as in Fig.~\ref{fig:template} and refined action that serve as the learning target for the Stage-2 critique. 

We consider two configurations in training: (1) \emph{Base CriticDrive}, which emphasizes MGS and mixes in 15\% unperturbed GT data to teach the critic when to output near-zero deltas, stabilizing $\beta$ on in-distribution samples; and (2) \emph{Full CriticDrive}, which samples (MGS+GT) and EPAS in 1:1 ratio to maximize exposure to diverse high-risk trajectories. Please refer to Table~\ref{tab:ablation_module} for the experimental results.

\section{Experiment}
\label{sec:exp}

\begin{figure*}[t]
    \centering
    \includegraphics[width=0.99\linewidth]{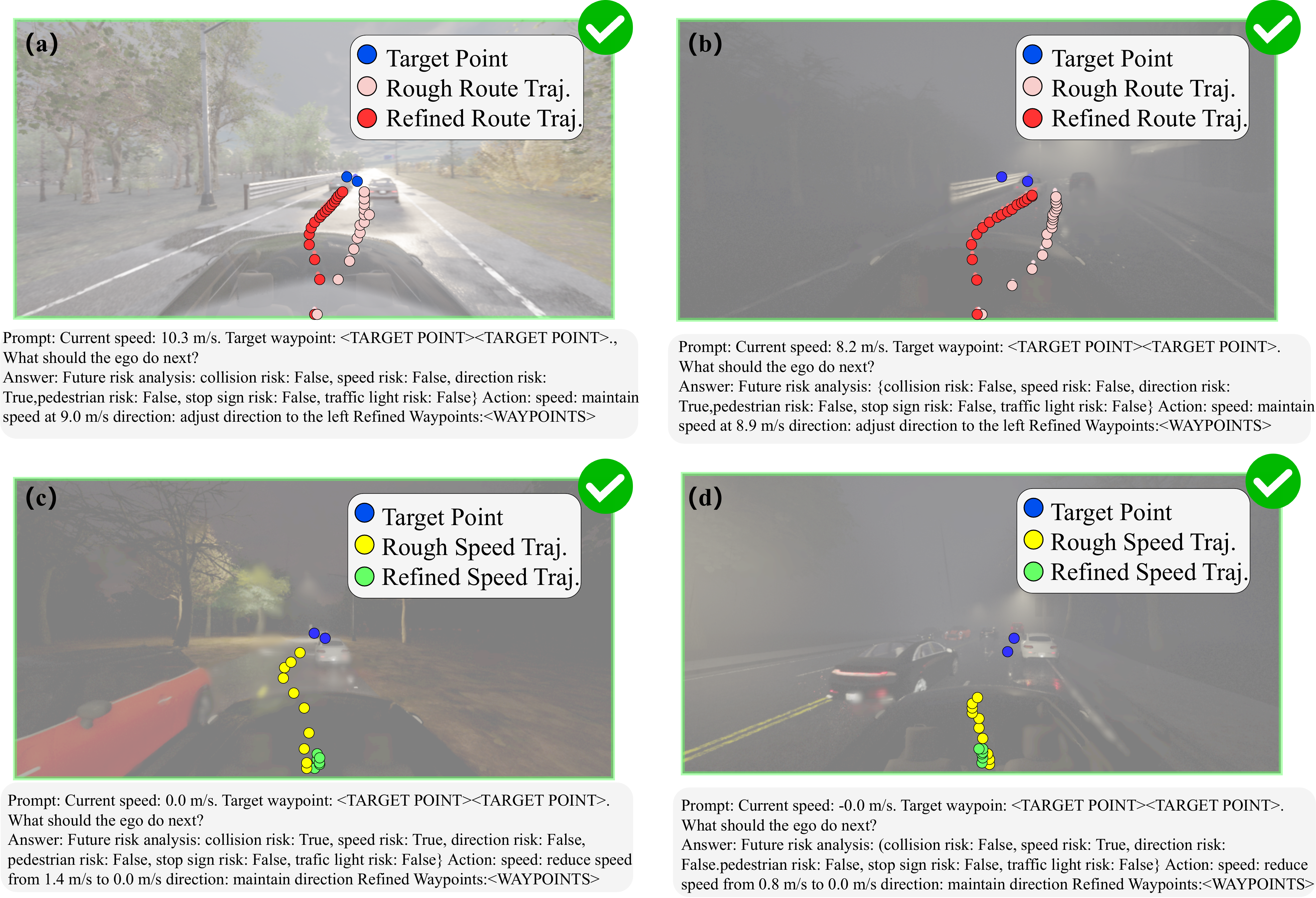}
    \caption{The Visualization of CriticVLA's refinement effect. The two blue points represent the current and next target points, respectively: (a) and (b) demonstrate the refinement of route waypoints leading to successful correction of steering to the highway entrance, (c) and (d) illustrate the adjustment of speed waypoints that control the future speed to avoid collision. The gray boxes display generated critiques. }
    \label{fig:vis}
\end{figure*}

\subsection{Experiment Settings}

In this section, we evaluate all methods on Bench2Drive~\cite{jia2024bench2drive}, a widely used challenging \emph{closed-loop} autonomous driving benchmark built on CARLA v0.9.15. Bench2Drive provides 220 test routes, each approximately 150 meters in length, covering diverse scenarios such as dense urban traffic, unprotected turns, highway merging, etc. Unlike open-loop benchmarks, the closed-loop setup requires the agent to continuously interact with other traffic participants and to update its decisions based on the evolving scene state. This makes it particularly suitable for assessing whether a critic can reliably detect potential risks and refine trajectories in highly interactive situations. 
Following the official protocol, we report four standard metrics: \emph{Driving Score} (DS) and \emph{Success Rate} (SR) as primary indicators of task completion and safety, and \emph{Efficiency} (speed performance) and \emph{Comfortness} (motion smoothness) as secondary indicators of driving quality. 

\subsection{Main Results} 
Table~\ref{tab:sota} compares CriticVLA with previous methods on Bench2Drive. CriticVLA attains a new state-of-the-art performance, reaching a Driving Score of 88.02 and a Success Rate of 73.33\%. 
Crucially, these gains in safety and task completion are achieved without sacrificing driving quality. Beyond primary metrics, CriticVLA also exhibits strong Efficiency and competitive Comfortness, indicating that the proposed ``judge, then drive'' paradigm effectively mitigates the typical safety--efficiency trade-off in end-to-end models. This balanced behavior is enabled by our two-stage architecture, where decoupled refinement for speed and direction allows the critic to provide precise, actionable guidance rather than vague commands (e.g., ``be careful''). As a result, the model can execute complex maneuvers while maintaining safe interactions and reasonable velocity.

Fig.~\ref{fig:vis} provides qualitative evidence of the refinement ability. The visualizations show that our critic model successfully detects and rectifies significant errors in the initial rough trajectories. 
These results confirm that our method can effectively reason about and correct both lateral risk and collision risk in closed-loop driving.

\renewcommand{\arraystretch}{1.1}
\begin{table}[t]
\centering
\caption{\textbf{ Ablation of critic model} 
on Bench2Drive.}
\label{tab:ablation_module}
\vspace{-6pt}
\scalebox{0.87}{
\begin{tabular}{l|cc}
\hline
Model Variant & DS $\uparrow$ & SR (\%) $\uparrow$  \\
\hline
Stage-1 only & 87.48 (\small{$\pm$0.72}) & 70.6 (\small{$\pm$1.45}) \\
Implicit critic & 87.49 (\small{$\pm$0.54}) & 71.06 (\small{$\pm$0.53}) \\ 
CriticVLA (\textit{Base CriticDrive}) & 87.97 (\small{$\pm$0.29}) & 72.27 (\small{$\pm$0.46})  \\ 
CriticVLA (\textit{Full CriticDrive}) & \textbf{88.02 (\small{$\pm$0.17})}
& \textbf{73.33 (\small$\pm$0.27)}  \\
\hline
\end{tabular}
}   
\end{table}

\begin{table*}[]
\centering
\caption{\textbf{Multi-Ability Results on Bench2Drive Closed-loop Benchmark.} C/L refers to camera/LiDAR, * denote using expert feature distillation. We conduct three independent trials with different seeds and report the mean and variance of the results.}
\label{tab:multi_ability}
\scalebox{0.75}{
% \begin{tabular}{>{\kern-\tabcolsep}*{8}{>$c<$}<{\kern-\tabcolsep}}
\begin{tabular}{@{}lcccccccc@{}}
\toprule
\multirow{2}{*}{\textbf{Method}} & \multirow{2}{*}{\textbf{Modality}} & \multirow{2}{*}{\textbf{Venue}} & \multicolumn{6}{c}{\textbf{Ability (\%) $\uparrow$}}  \cr
\cmidrule(lr){4-9}

& & & Merging & Overtaking & Emergency Brake & Give Way & Traffic Sign & Mean  \cr 
\midrule
TCP*~\cite{TCP} & C & {NeurIPS' 22} & 16.18 & 20.00 & 20.00 & 10.00 & 6.99 & 14.63  \cr
TCP-traj*~ & C & {NeurIPS' 22} & 8.89 & 24.29 & 51.67 & 40.00 & 46.28 & 34.22  \cr
UniAD-Base~\cite{hu2023uniad} & C & {CVPR' 23} & 14.10 & 17.78 & 21.67 & 10.00 & 14.21 & 15.55  \cr
ThinkTwice*~\cite{jia2023think} & C & {CVPR' 23} & 27.38 & 18.42 & 35.82 & 50.00 & 54.23 & 37.17  \cr
VAD~\cite{jiang2023vad} & C & {ICCV' 23} & 8.11 & 24.44 & 18.64 & 20.00 & 19.15 & 18.07  \cr
DriveAdaptor*~\cite{jia2023driveadapter} & C+L & {ICCV' 23} & 28.82 & 26.38 & 48.76 & 50.00 & 56.43 & 42.08  \cr
DriveTrans~\cite{jia2025drivetransformer} & C & {ICLR' 25} & 17.57 & 35.00 & 48.36 & 40.00 & 52.10 & 38.60  \cr
TransFuser++~\cite{zimmerlin2024hidden} & C+L & {arXiv' 24} & 58.75 & 57.77 & \textbf{83.33} & 40.00 & 82.11 & 64.39  \cr
ORION~\cite{fu2025orion} & C & {ICCV' 25} & 25.00 & 71.11 & 78.33 & 30.00 & 69.15 & 54.72  \cr
Simlingo~\cite{renz2025simlingo}  & C & {CVPR' 25} & 54.01 & 57.04 & \textbf{88.33} & \textbf{53.33} & \textbf{82.45} & 67.03  \cr
\midrule
\cellcolor[HTML]{F2F2F2} CriticVLA~(\textit{Ours}) & \cellcolor[HTML]{F2F2F2} C & \cellcolor[HTML]{F2F2F2} - & \cellcolor[HTML]{F2F2F2} \textbf{61.28\small{$\pm$1.30}} & \cellcolor[HTML]{F2F2F2} \textbf{76.30\small{$\pm$1.28}} & \cellcolor[HTML]{F2F2F2} \textbf{88.33\small{$\pm$0.00}} & \cellcolor[HTML]{F2F2F2} 50.00\small{$\pm$0.00} & \cellcolor[HTML]{F2F2F2} 81.06\small{$\pm$0.91} & \cellcolor[HTML]{F2F2F2} \textbf{71.39\small{$\pm$0.26}} \cr
\bottomrule
\end{tabular}
}
\end{table*}

\section{Analysis}
\label{sec:ana}
In this section, we conduct a deeper analysis of the proposed CriticVLA, with a particular focus on its ability to refine actions in challenging scenarios. 

\begin{mdframed}[backgroundcolor=gray!8]
\begin{minipage}{\linewidth}
(1) How to design a more effective critic model?
\end{minipage}
\end{mdframed}

Table~\ref{tab:ablation_module} presents an ablation of our critic design on Bench2Drive, comparing four variants that progressively introduce (1) an explicit language critic and (2) Extra Perturbation-Augmented Subset from CriticDrive.
The first row is the Stage-1 only baseline, whereas Implicit critic~(Row 2) reuses the Stage-1 architecture to take its own coarse trajectory as input and directly output a refined trajectory, without generating any language critique. This configuration can be viewed as an implicit critic without language. 

\textbf{Importance of language critic:} Row~2 and Row~3 share the same training data and hyperparameter, and the only difference between these two variants is the present of explicit language critic. The language critic can articulate collision, speed, and direction risks and provide targeted corrections, rather than relying on implicit feature-level adjustments. Their relative SR improvement to Stage-1 (+0.46 vs. +1.67) supports our design choice that an explicit, language-based critic is more effective than an implicit critic without language.
This also indicates that making the refinement signal linguistically explicit helps the model focus on safety-critical failure modes and deliver more targeted trajectory corrections under the same training setup.

\textbf{Necessity of using \textit{Full CriticDrive}:} combining the Extra Perturbation-Augmented Subset with the Model-Generated Subset (Row 4) yields the best overall performance.
EPAS provides more diverse critiques, including aggressive speed profiles, lane intrusions, and synthetic collision trajectories, which broadens the critic’s understanding of risk and enhance its ability to refine actions. 
Consequently, CriticVLA trained on \textit{Full CriticDrive} attains higher SR with remarkably low variance across runs, indicating stronger robustness in closed-loop interactions with other traffic participants. This empirically validates the insights from Theorem \ref{theo:1} regarding the necessity of expanding the action set $T$ to achieve a small $\rho_0$ while maintaining a positive improvement ratio $\beta$.

\begin{mdframed}[backgroundcolor=gray!8]
\begin{minipage}{\linewidth}
(2) What difficult scenarios does the proposed CriticVLA specifically address?
\end{minipage}
\end{mdframed}
Table~\ref{tab:multi_ability} evaluates five advanced driving capabilities and compares CriticVLA with prior methods. CriticVLA achieves particularly large gains on \textbf{Merging} and \textbf{Overtaking}, improving success rates by 7.27\% and 19.26\% over SimLingo, respectively. Both abilities pertain to highly interactive scenarios 
where failures often stem from rear-end collision or unsafe lateral maneuvers. Under these scenarios, CriticVLA’s dynamic risk analysis and subsequent action refinement play a crucial role in coping with interaction without becoming excessively aggressive or conservative, thereby resulting in a safe and efficient driving policy.

To get a better understanding of what scenarios CriticVLA specifically address, we further perform a scenario‑wise analysis by grouping Bench2Drive routes according to their scenario type and report per‑scenario success rates in Fig.\ref{fig:perdiff}. 
CriticVLA achieves the largest relative gain in scenarios that demand strict safety compliance and precise longitudinal control. In \textbf{InterurbanActorFlow (SR 20\%$\rightarrow$100\%)}, longitudinal risk is identified and appropriate speed adjustments are suggested to help the agent match the surrounding traffic flow instead of either stalling or merging aggressively. For \textbf{NonSignalizedJunctionLeft/RightTurn (SR 60\%$\rightarrow$100\%)}, lateral risk analysis detects lateral deviations that would lead to cutting corners or drifting into opposite lanes and refines the trajectory back to the lane. 
In most of these scenarios, \emph{Full CriticDrive} further improves success rates over the \emph{Base CriticDrive} variant, indicating that exposing the critic to a richer set of high-risk perturbed trajectories translates into better generalization on challenging cases.

\begin{figure}%[t]
    \centering
    \includegraphics[width=1\linewidth]{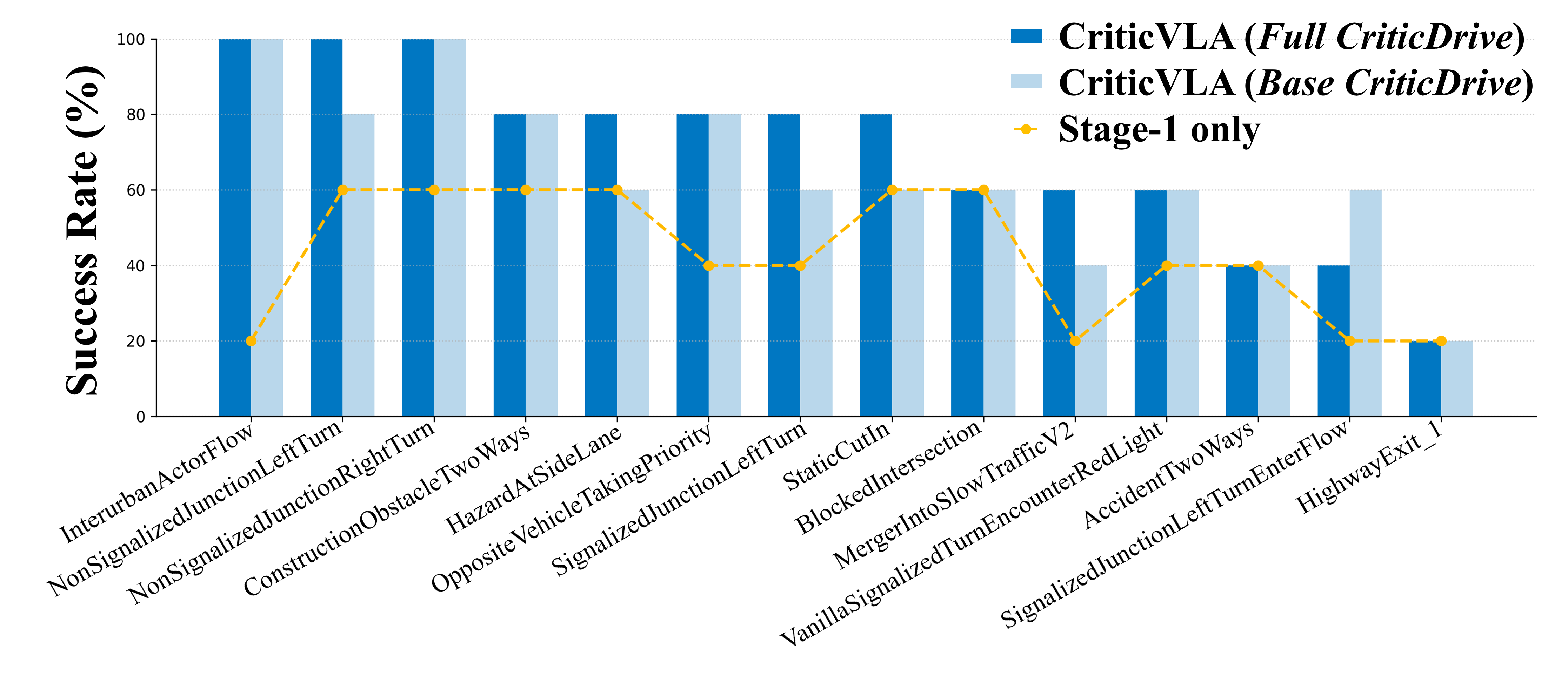}
    \vspace{-0.5cm}
    \caption{The performance of CriticVLA in difficult scenarios. }
    \label{fig:perdiff}
\end{figure}

\begin{mdframed}[backgroundcolor=gray!8]
\begin{minipage}{\linewidth}
(3) Can the proposed CriticVLA be applied iteratively to further optimize driving decisions?
\end{minipage}
\end{mdframed}
To enable fine-grained diagnosis, we construct a compact core-set $D_{\text{core}}$ from the full Bench2Drive routes by retaining the most challenging route from each scenario (See Appendix for details), yielding 44 routes in total. This core-set preserves the diversity while serving as a rigorous proxy for evaluating driving policies. Consequently, all subsequent analysis is conducted on it.

Since Sections~\ref{sec:method} and~\ref{sec:exp} have established the effectiveness of CriticVLA for one-step refinement of action, we now analyze whether it can be applied iteratively.
\begin{theorem}\label{theo:2}
 For the initial action $A_0$ under given input $V$ and $L$ in Theorem \ref{theo:1}, Let $A_k = C(A_{k-1}), k \geq 1$, and $\rho_k = \rho(A_k)$ denote the distance from $A_k$ to the training set $T$ like Eq.~(\ref{eq:rho}). A sufficient condition for the sequence $\{Q(A_k)\}$ to be monotonically increasing is that  
\begin{equation}
    \rho_k \leq c_\beta\cdot [Q^* - Q(A_k)],
    \label{eq:Qa2}
\end{equation}
where $Q^*$ is the upper bound value for given $V,L$ and the corresponding optimal action as defined in Eq.~(\ref{eq:qstar}). $c_\beta>0$ is the constant related to $\beta$.
\end{theorem}
\begin{proof}
	(See Appendix for details).\qedhere
\end{proof}
If the proposed CriticVLA can iteratively refine the rough driving $A_0$, this implies that
\begin{equation}
    Q(A_0) \leq Q(A_1)  \leq \dots \leq Q(A_\infty) \to Q^*, 
    \label{eq:qq}
\end{equation}

i.e., $\{Q(A_k)\}$ is monotonically increasing and converges to the upper bound $Q^*$. Theorem~\ref{theo:2} shows that this requires Eq.~(\ref{eq:Qa2}) to hold, which in turn implies
\begin{equation}
    \min_{A \in T} d(A_k, A) \leq c_\beta \cdot [Q^* - Q(A_k)] \leq c_\beta\epsilon.
\end{equation}
for any $\epsilon > 0$ and all sufficiently large $k$. Hence, repeated refinement would force $\min_{A \in T} d(A_k, A) \to 0$, meaning the training set $T$ must effectively cover all refined actions $A_k$ and even the optimal action $A^*$. In practice, constructing such a training set for autonomous driving is infeasible, even with large synthetic datasets like CriticDrive, as it would require $T$ to nearly cover the entire test distribution.

The empirical results in Table~\ref{tab:ablation_multi_turn} show that the first refinement step produces a substantial performance gain, while the second step offers only a marginal additional improvement with diminishing returns, consistent with the theoretical analysis above. Balancing effectiveness and computational cost, we recommend one-step refinement with CriticVLA as the preferred operating mode.

\begin{table}[h]
\centering
\caption{\textbf{Multi-step refinements} 
on core-set.}
\vspace{-6pt}
\scalebox{0.99}{
\begin{tabular}{l|cc}
\hline
Multi-step & DS $\uparrow$ & SR(\%) $\uparrow$  \\
\hline
0 & 82.58 &  61.36 \\
1 (ours) & 88.24 & 72.73 \\
2 & \textbf{88.75} & \textbf{75.00} \\
3 & 88.13 & 70.45 \\
\hline
\end{tabular}
}
\label{tab:ablation_multi_turn}
\end{table}
\vspace{-0.1cm}

\begin{mdframed}[backgroundcolor=gray!8]
\begin{minipage}{\linewidth}
(4) Is the main Theorem empirically valid in autonomous driving scenarios?
\end{minipage}
\end{mdframed}
In Theorem \ref{theo:1}, we revealed that both better quality of the rough $A_0$ and stronger refinement ability of the critic model can enhance the performance of our proposed CriticVLA. In this section, we will further conduct numerical verification of these conclusions in the autonomous driving scenario based on core-set $D_{\text{core}}$. Specifically,

\textbf{(a) Quality of $A_0$}: We consider $\tilde{A_0} = A_0 + \delta a$ to simulate the quality of rough decisions, where $\delta a \sim \mathcal{N}(\mathbf{0},\sigma \mathbf{I})$ and $\sigma \in \mathbb{R}^+$ represents the level of perturbation applied to rough trajectories. This additive Gaussian perturbation provides a controlled way to mimic errors from Stage-1 trajectory generation (e.g., imperfect perception or planning uncertainty). The larger the value of $\sigma$, the worse the quality of $\tilde{A_0}$, and the more difficult it becomes for the critic to recover a safe and feasible trajectory within one refinement step. The results in Table~\ref{tab:ablation_rough} illustrate that the quality of $A_0$ is crucial for CriticVLA and directly impacts the final closed-loop driving performance.

\textbf{(b) Refinement ability of the critic model}: During training, each epoch samples a fixed-size minibatch from the \textit{Full CriticDrive} dataset, so CriticVLA is exposed to more training samples by increasing the number of epochs, effectively improving the refinement ability. The results in Table \ref{tab:ablation_generalization} confirm that the refinement ability of the critic model is crucial in the framework we proposed.

\begin{table}[h]
\centering
\caption{\textbf{ Ablation on Quality of $A_0$} 
on core-set.}
\vspace{-6pt}
\scalebox{0.99}{
\begin{tabular}{c|cc}
\hline
Quality of $A_0$ & DS $\uparrow$ & SR(\%) $\uparrow$  \\
\hline
 $\sigma=1.0$ & 85.18 & 62.79  \\
$\sigma=0.5$ & 86.57 & 70.45 \\
without noise & \textbf{88.24} & \textbf{72.73} \\
\hline
\end{tabular}
}
\label{tab:ablation_rough}
\end{table}

\begin{table}[h]
\centering
\caption{\textbf{ Ablation of refinement ability of critic model} 
on core-set. ``e" denotes epoch.}
\vspace{-6pt}
\scalebox{0.99}{
\begin{tabular}{c|cc}
\hline
Critic model & DS $\uparrow$ & SR(\%) $\uparrow$  \\
\hline
base (e=0) & 82.58 & 61.36 \\
% critic (e=1) & 79.70 & 56.82  \\
critic (e=5) & 84.04 & 65.91 \\
critic (e=9) & 86.43 & 70.45  \\
critic (e=13) & \textbf{88.24} & \textbf{72.73} \\
\hline
\end{tabular}
}
\label{tab:ablation_generalization}
\end{table}

\section{Conclusion}
\label{sec:conclusion}

In this work, we introduce CriticVLA, a novel two-stage framework that redefines the role of Vision-Language-Action (VLA) models in autonomous driving. We move beyond simple action generation by explicitly leveraging the VLA's critic capability. Our paradigm first generates an initial trajectory and then refines it via VLA-based multimodal critique, leading to high-quality driving behaviors. To support this, we construct CriticDrive, a large-scale synthetic dataset with 12.9 million annotated trajectories to enhance critic generation. Extensive closed-loop evaluations on Bench2Drive demonstrate CriticVLA's superior performance over state-of-the-art baselines. Our theoretical and empirical analyses confirm that the critic mechanism significantly improves decision quality, establishing a promising ``judge, then drive" paradigm for complex real-world driving.

\clearpage

\section*{Impact Statement}
This work advances the field of Machine Learning by proposing a novel critic-driven paradigm for autonomous driving. The developed framework, CriticVLA, and its accompanying dataset aim to enhance the safety and decision-making reliability of autonomous vehicles by enabling self-critical refinement of driving trajectories. While the direct societal impact is the potential improvement of autonomous driving systems, broader deployment would necessitate careful consideration of real-world safety validation, ethical implications of automated decision-making in critical scenarios, and the societal consequences of widespread autonomous vehicle adoption.

% In the unusual situation where you want a paper to appear in the
% references without citing it in the main text, use \nocite
\nocite{langley00}

\bibliography{example_paper}

@String(ECCV= {Eur. Conf. Comput. Vis.})

@String(ICLR = {Int. Conf. Learn. Represent.})

@String(ECCV  = {ECCV})

@String(ICLR  = {ICLR})

@inproceedings{chen2024internvl,
  title={Internvl: Scaling up vision foundation models and aligning for generic visual-linguistic tasks},
  author={Chen, Zhe and Wu, Jiannan and Wang, Wenhai and Su, Weijie and Chen, Guo and Xing, Sen and Zhong, Muyan and Zhang, Qinglong and Zhu, Xizhou and Lu, Lewei and others},
  booktitle={Proceedings of the IEEE/CVF conference on computer vision and pattern recognition},
  pages={24185--24198},
  year={2024}
}

@inproceedings{carion2020end,
  title={End-to-end object detection with transformers},
  author={Carion, Nicolas and Massa, Francisco and Synnaeve, Gabriel and Usunier, Nicolas and Kirillov, Alexander and Zagoruyko, Sergey},
  booktitle={European conference on computer vision},
  pages={213--229},
  year={2020},
  organization={Springer}
}

@article{hu2022lora,
  title={Lora: Low-rank adaptation of large language models.},
  author={Hu, Edward J and Shen, Yelong and Wallis, Phillip and Allen-Zhu, Zeyuan and Li, Yuanzhi and Wang, Shean and Wang, Lu and Chen, Weizhu and others},
  journal={ICLR},
  volume={1},
  number={2},
  pages={3},
  year={2022}
}

@inproceedings{zhong2024moextend,
  title={Moextend: Tuning new experts for modality and task extension},
  author={Zhong, Shanshan and Gao, Shanghua and Huang, Zhongzhan and Wen, Wushao and {\v{Z}}itnik, Marinka and Zhou, Pan},
  booktitle={Proceedings of the 62nd Annual Meeting of the Association for Computational Linguistics (Volume 4: Student Research Workshop)},
  pages={494--505},
  year={2024}
}

@article{huang2025minilongbench,
  title={MiniLongBench: The Low-cost Long Context Understanding Benchmark for Large Language Models},
  author={Huang, Zhongzhan and Ling, Guoming and Zhong, Shanshan and Wu, Hefeng and Lin, Liang},
  journal={arXiv preprint arXiv:2505.19959},
  year={2025}
}

@article{wang2025llava,
  title={Llava-critic-r1: Your critic model is secretly a strong policy model},
  author={Wang, Xiyao and Li, Chunyuan and Yang, Jianwei and Zhang, Kai and Liu, Bo and Xiong, Tianyi and Huang, Furong},
  journal={arXiv preprint arXiv:2509.00676},
  year={2025}
}

@inproceedings{zhong2024let,
  title={Let's think outside the box: Exploring leap-of-thought in large language models with creative humor generation},
  author={Zhong, Shanshan and Huang, Zhongzhan and Gao, Shanghua and Wen, Wushao and Lin, Liang and Zitnik, Marinka and Zhou, Pan},
  booktitle={Proceedings of the IEEE/CVF Conference on Computer Vision and Pattern Recognition},
  pages={13246--13257},
  year={2024}
}

@article{liu2025bridgedrive,
  title={BridgeDrive: Diffusion Bridge Policy for Closed-Loop Trajectory Planning in Autonomous Driving},
  author={Liu, Shu and Chen, Wenlin and Li, Weihao and Wang, Zheng and Yang, Lijin and Huang, Jianing and Zhang, Yipin and Huang, Zhongzhan and Cheng, Ze and Yang, Hao},
  journal={arXiv preprint arXiv:2509.23589},
  year={2025}
}

@article{huang2025causality,
  title={A causality-aware paradigm for evaluating creativity of multimodal large language models},
  author={Huang, Zhongzhan and Zhong, Shanshan and Zhou, Pan and Gao, Shanghua and Zitnik, Marinka and Lin, Liang},
  journal={IEEE Transactions on Pattern Analysis and Machine Intelligence},
  year={2025},
  publisher={IEEE}
}

@article{jia2025drivetransformer,
  title={Drivetransformer: Unified transformer for scalable end-to-end autonomous driving},
  author={Jia, Xiaosong and You, Junqi and Zhang, Zhiyuan and Yan, Junchi},
  journal={arXiv preprint arXiv:2503.07656},
  year={2025}
}

@inproceedings{lee2024prometheus,
  title={Prometheus-vision: Vision-language model as a judge for fine-grained evaluation},
  author={Lee, Seongyun and Kim, Seungone and Park, Sue and Kim, Geewook and Seo, Minjoon},
  booktitle={Findings of the association for computational linguistics ACL 2024},
  pages={11286--11315},
  year={2024}
}

@article{zheng2023judging,
  title={Judging llm-as-a-judge with mt-bench and chatbot arena},
  author={Zheng, Lianmin and Chiang, Wei-Lin and Sheng, Ying and Zhuang, Siyuan and Wu, Zhanghao and Zhuang, Yonghao and Lin, Zi and Li, Zhuohan and Li, Dacheng and Xing, Eric and others},
  journal={Advances in neural information processing systems},
  volume={36},
  pages={46595--46623},
  year={2023}
}

@inproceedings{liu2024improved,
  title={Improved baselines with visual instruction tuning},
  author={Liu, Haotian and Li, Chunyuan and Li, Yuheng and Lee, Yong Jae},
  booktitle={Proceedings of the IEEE/CVF conference on computer vision and pattern recognition},
  pages={26296--26306},
  year={2024}
}

@article{pan2024human,
  title={Human-Centered Design Recommendations for LLM-as-a-judge},
  author={Pan, Qian and Ashktorab, Zahra and Desmond, Michael and Cooper, Martin Santillan and Johnson, James and Nair, Rahul and Daly, Elizabeth and Geyer, Werner},
  journal={arXiv preprint arXiv:2407.03479},
  year={2024}
}

@article{huang2025routereval,
  title={Routereval: A comprehensive benchmark for routing llms to explore model-level scaling up in llms},
  author={Huang, Zhongzhan and Ling, Guoming and Lin, Yupei and Chen, Yandong and Zhong, Shanshan and Wu, Hefeng and Lin, Liang},
  journal={arXiv preprint arXiv:2503.10657},
  year={2025}
}

@inproceedings{zhong2022cem,
  title={Cem: Machine-human chatting handoff via causal-enhance module},
  author={Zhong, Shanshan and Qin, Jinghui and Huang, Zhongzhan and Li, Daifeng},
  booktitle={Proceedings of the 2022 Conference on Empirical Methods in Natural Language Processing},
  pages={3242--3253},
  year={2022}
}

@inproceedings{liu2025associam,
  title={AssoCiAm: A Benchmark for Evaluating Association Thinking while Circumventing Ambiguity},
  author={Liu, Yifan and Zhao, Wenkuan and Zhong, Shanshan and Qin, Jinghui and Liang, Mingfu and Huang, Zhongzhan and Wen, Wushao},
  booktitle={Proceedings of the 2025 Conference on Empirical Methods in Natural Language Processing},
  pages={5203--5219},
  year={2025}
}

@inproceedings{renz2025simlingo,
  title={Simlingo: Vision-only closed-loop autonomous driving with language-action alignment},
  author={Renz, Katrin and Chen, Long and Arani, Elahe and Sinavski, Oleg},
  booktitle={Proceedings of the Computer Vision and Pattern Recognition Conference},
  pages={11993--12003},
  year={2025}
}

@article{sapkota2025vla,
  title={Vision-language-action models: Concepts, progress, applications and challenges},
  author={Sapkota, Ranjan and Cao, Yang and Roumeliotis, Konstantinos I and Karkee, Manoj},
  journal={arXiv preprint arXiv:2505.04769},
  year={2025}
}

@inproceedings{jia2023think,
  title={Think twice before driving: Towards scalable decoders for end-to-end autonomous driving},
  author={Jia, Xiaosong and Wu, Penghao and Chen, Li and Xie, Jiangwei and He, Conghui and Yan, Junchi and Li, Hongyang},
  booktitle={Proceedings of the IEEE/CVF Conference on Computer Vision and Pattern Recognition},
  pages={21983--21994},
  year={2023}
}

@inproceedings{hu2023uniad,
  title={Planning-oriented autonomous driving},
  author={Hu, Yihan and Yang, Jiazhi and Chen, Li and Li, Keyu and Sima, Chonghao and Zhu, Xizhou and Chai, Siqi and Du, Senyao and Lin, Tianwei and Wang, Wenhai and others},
  booktitle={Proceedings of the IEEE/CVF conference on computer vision and pattern recognition},
  pages={17853--17862},
  year={2023}
}

@inproceedings{jiang2023vad,
  title={Vad: Vectorized scene representation for efficient autonomous driving},
  author={Jiang, Bo and Chen, Shaoyu and Xu, Qing and Liao, Bencheng and Chen, Jiajie and Zhou, Helong and Zhang, Qian and Liu, Wenyu and Huang, Chang and Wang, Xinggang},
  booktitle={Proceedings of the IEEE/CVF International Conference on Computer Vision},
  pages={8340--8350},
  year={2023}
}

@inproceedings{zheng2024genad,
  title={Genad: Generative end-to-end autonomous driving},
  author={Zheng, Wenzhao and Song, Ruiqi and Guo, Xianda and Zhang, Chenming and Chen, Long},
  booktitle={European Conference on Computer Vision},
  pages={87--104},
  year={2024},
  organization={Springer}
}

@article{ma2024surveyvla,
  title={A survey on vision-language-action models for embodied ai},
  author={Ma, Yueen and Song, Zixing and Zhuang, Yuzheng and Hao, Jianye and King, Irwin},
  journal={arXiv preprint arXiv:2405.14093},
  year={2024}
}

@inproceedings{nuscenes2019,
  title={nuscenes: A multimodal dataset for autonomous driving},
  author={Caesar, Holger and Bankiti, Varun and Lang, Alex H and Vora, Sourabh and Liong, Venice Erin and Xu, Qiang and Krishnan, Anush and Pan, Yu and Baldan, Giancarlo and Beijbom, Oscar},
  booktitle={Proceedings of the IEEE/CVF conference on computer vision and pattern recognition},
  pages={11621--11631},
  year={2020}
}

@inproceedings{waymo2024open,
  title={Scalability in perception for autonomous driving: Waymo open dataset},
  author={Sun, Pei and Kretzschmar, Henrik and Dotiwalla, Xerxes and Chouard, Aurelien and Patnaik, Vijaysai and Tsui, Paul and Guo, James and Zhou, Yin and Chai, Yuning and Caine, Benjamin and others},
  booktitle={Proceedings of the IEEE/CVF conference on computer vision and pattern recognition},
  pages={2446--2454},
  year={2020}
}

@article{wilson2023argoverse2,
  title={Argoverse 2: Next generation datasets for self-driving perception and forecasting},
  author={Wilson, Benjamin and Qi, William and Agarwal, Tanmay and Lambert, John and Singh, Jagjeet and Khandelwal, Siddhesh and Pan, Bowen and Kumar, Ratnesh and Hartnett, Andrew and Pontes, Jhony Kaesemodel and others},
  journal={arXiv preprint arXiv:2301.00493},
  year={2023}
}

@article{neurips2024drivingdojo,
  title={Drivingdojo dataset: Advancing interactive and knowledge-enriched driving world model},
  author={Wang, Yuqi and Cheng, Ke and He, Jiawei and Wang, Qitai and Dai, Hengchen and Chen, Yuntao and Xia, Fei and Zhang, Zhao-Xiang},
  journal={Advances in Neural Information Processing Systems},
  volume={37},
  pages={13020--13034},
  year={2024}
}

@article{fu2025orion,
  title={Orion: A holistic end-to-end autonomous driving framework by vision-language instructed action generation},
  author={Fu, Haoyu and Zhang, Diankun and Zhao, Zongchuang and Cui, Jianfeng and Liang, Dingkang and Zhang, Chong and Zhang, Dingyuan and Xie, Hongwei and Wang, Bing and Bai, Xiang},
  journal={arXiv preprint arXiv:2503.19755},
  year={2025}
}

@inproceedings{chi2024covla,
  title={{CoVLA}: Comprehensive Vision-Language-Action Dataset for Autonomous Driving},
  author={Chi, Haohan and Gao, Huan-ang and Liu, Ziming and others},
  booktitle={WACV},
  year={2025}
}

@inproceedings{dosovitskiy2017carla,
  title={CARLA: An open urban driving simulator},
  author={Dosovitskiy, Alexey and Ros, German and Codevilla, Felipe and Lopez, Antonio and Koltun, Vladlen},
  booktitle={Conference on robot learning},
  pages={1--16},
  year={2017},
  organization={PMLR}
}

@article{jiang2025diffvla,
  title={DiffVLA: Vision-Language Guided Diffusion Planning for Autonomous Driving},
  author={Jiang, Anqing and Gao, Yu and Sun, Zhigang and others},
  journal={arXiv preprint arXiv:2505.19381},
  year={2025}
}

@article{zhang2025safeauto,
  title={SafeAuto: Knowledge-Enhanced Safe Autonomous Driving with Multimodal Foundation Models},
  author={Zhang, Jiawei and Yang, Xuan and Wang, Taiqi and Yao, Yu and Petiushko, Aleksandr and Li, Bo},
  journal={arXiv preprint arXiv:2503.00211},
  year={2025}
}

@article{sima2024drivelm,
  title={DriveLM: Driving with Graph Visual Question Answering},
  author={Sima, Chonghao and Renz, Katrin and Chitta, Kashyap and others},
  journal={ECCV},
  year={2024}
}

@article{liu2023visual,
  title={Visual instruction tuning},
  author={Liu, Haotian and Li, Chunyuan and Wu, Qingyang and Lee, Yong Jae},
  journal={Advances in neural information processing systems},
  volume={36},
  pages={34892--34916},
  year={2023}
}

@article{yuan2024rag,
  title={Rag-driver: Generalisable driving explanations with retrieval-augmented in-context learning in multi-modal large language model},
  author={Yuan, Jianhao and Sun, Shuyang and Omeiza, Daniel and Zhao, Bo and Newman, Paul and Kunze, Lars and Gadd, Matthew},
  journal={arXiv preprint arXiv:2402.10828},
  year={2024}
}

@article{jia2024bench2drive,
  title={Bench2drive: Towards multi-ability benchmarking of closed-loop end-to-end autonomous driving},
  author={Jia, Xiaosong and Yang, Zhenjie and Li, Qifeng and Zhang, Zhiyuan and Yan, Junchi},
  journal={arXiv preprint arXiv:2406.03877},
  year={2024}
}

@article{zhou2024vision,
  title={Vision language models in autonomous driving: A survey and outlook},
  author={Zhou, Xingcheng and Liu, Mingyu and Yurtsever, Ekim and Zagar, Bare Luka and Zimmer, Walter and Cao, Hu and Knoll, Alois C},
  journal={IEEE Transactions on Intelligent Vehicles},
  year={2024},
  publisher={IEEE}
}

@article{cui2025chain,
  title={Chain-of-Thought for Autonomous Driving: A Comprehensive Survey and Future Prospects},
  author={Cui, Yixin and Lin, Haotian and Yang, Shuo and Wang, Yixiao and Huang, Yanjun and Chen, Hong},
  journal={arXiv preprint arXiv:2505.20223},
  year={2025}
}

@article{wang2024drivecot,
  title={Drivecot: Integrating chain-of-thought reasoning with end-to-end driving},
  author={Wang, Tianqi and Xie, Enze and Chu, Ruihang and Li, Zhenguo and Luo, Ping},
  journal={arXiv preprint arXiv:2403.16996},
  year={2024}
}

@article{kim2024openvla,
  title={Openvla: An open-source vision-language-action model},
  author={Kim, Moo Jin and Pertsch, Karl and Karamcheti, Siddharth and Xiao, Ted and Balakrishna, Ashwin and Nair, Suraj and Rafailov, Rafael and Foster, Ethan and Lam, Grace and Sanketi, Pannag and others},
  journal={arXiv preprint arXiv:2406.09246},
  year={2024}
}

@article{achiam2023gpt,
  title={Gpt-4 technical report},
  author={Achiam, Josh and Adler, Steven and Agarwal, Sandhini and Ahmad, Lama and Akkaya, Ilge and Aleman, Florencia Leoni and Almeida, Diogo and Altenschmidt, Janko and Altman, Sam and Anadkat, Shyamal and others},
  journal={arXiv preprint arXiv:2303.08774},
  year={2023}
}

@article{oquab2023dinov2,
  title={Dinov2: Learning robust visual features without supervision},
  author={Oquab, Maxime and Darcet, Timoth{\'e}e and Moutakanni, Th{\'e}o and Vo, Huy and Szafraniec, Marc and Khalidov, Vasil and Fernandez, Pierre and Haziza, Daniel and Massa, Francisco and El-Nouby, Alaaeldin and others},
  journal={arXiv preprint arXiv:2304.07193},
  year={2023}
}

@article{li2023driveLM,
  title={DriveLM: Driving with Graph Visual Question Answering},
  author={Li, Y. and Chen, L. and Han, X. and others},
  journal={arXiv preprint arXiv:2312.14140},
  year={2023}
}

@article{jiang2025survey,
  title={A Survey on Vision-Language-Action Models for Autonomous Driving},
  author={Jiang, Sicong and Huang, Zilin and Qian, Kangan and Luo, Ziang and Zhu, Tianze and Zhong, Yang and Tang, Yihong and Kong, Menglin and Wang, Yunlong and Jiao, Siwen and others},
  journal={arXiv preprint arXiv:2506.24044},
  year={2025}
}

@inproceedings{jia2023driveadapter,
  title={Driveadapter: Breaking the coupling barrier of perception and planning in end-to-end autonomous driving},
  author={Jia, Xiaosong and Gao, Yulu and Chen, Li and Yan, Junchi and Liu, Patrick Langechuan and Li, Hongyang},
  booktitle={Proceedings of the IEEE/CVF International Conference on Computer Vision},
  pages={7953--7963},
  year={2023}
}

@article{lee2023rlaif,
 author = {Lee, Harrison and Phatale, Samrat and Mansoor, Hassan and Mesnard, Thomas and Ferret, Johan and Lu, Kellie and Bishop, Colton and Hall, Ethan and Carbune, Victor and Rastogi, Abhinav and others},
 journal = {ArXiv preprint},
 title = {Rlaif: Scaling reinforcement learning from human feedback with ai feedback},
 url = {https://arxiv.org/abs/2309.00267},
 volume = {abs/2309.00267},
 year = {2023}
}

@inproceedings{sun2024salmon,
 author = {Sun, Zhiqing and Shen, Yikang and Zhang, Hongxin and Zhou, Qinhong and Chen, Zhenfang and Cox, David Daniel and Yang, Yiming and Gan, Chuang},
 booktitle = {The Twelfth International Conference on Learning Representations},
 title = {SALMON: Self-Alignment with Instructable Reward Models},
 year = {2024}
}

@inproceedings{bai2024benchmarking,
 author = {Yushi Bai and
Jiahao Ying and
Yixin Cao and
Xin Lv and
Yuze He and
Xiaozhi Wang and
Jifan Yu and
Kaisheng Zeng and
Yijia Xiao and
Haozhe Lyu and
Jiayin Zhang and
Juanzi Li and
Lei Hou},
 bibsource = {dblp computer science bibliography, https://dblp.org},
 booktitle = {Advances in Neural Information Processing Systems 36: Annual Conference
on Neural Information Processing Systems 2023, NeurIPS 2023, New Orleans,
LA, USA, December 10 - 16, 2023},
 editor = {Alice Oh and
Tristan Naumann and
Amir Globerson and
Kate Saenko and
Moritz Hardt and
Sergey Levine},
 title = {Benchmarking Foundation Models with Language-Model-as-an-Examiner},
 year = {2023}
}

@article{liang2023encouraging,
 author = {Liang, Tian and He, Zhiwei and Jiao, Wenxiang and Wang, Xing and Wang, Yan and Wang, Rui and Yang, Yujiu and Tu, Zhaopeng and Shi, Shuming},
 journal = {ArXiv preprint},
 title = {Encouraging divergent thinking in large language models through multi-agent debate},
 url = {https://arxiv.org/abs/2305.19118},
 volume = {abs/2305.19118},
 year = {2023}
}

@inproceedings{zhu2024starling,
 author = {Zhu, Banghua and Frick, Evan and Wu, Tianhao and Zhu, Hanlin and Ganesan, Karthik and Chiang, Wei-Lin and Zhang, Jian and Jiao, Jiantao},
 booktitle = {First Conference on Language Modeling},
 title = {Starling-7b: Improving helpfulness and harmlessness with rlaif},
 year = {2024}
}

@article{yang2023auto,
 author = {Yang, Hui and Yue, Sifu and He, Yunzhong},
 journal = {ArXiv preprint},
 title = {Auto-gpt for online decision making: Benchmarks and additional opinions},
 url = {https://arxiv.org/abs/2306.02224},
 volume = {abs/2306.02224},
 year = {2023}
}

@inproceedings{papineni2002bleu,
 address = {Philadelphia, Pennsylvania, USA},
 author = {Papineni, Kishore  and
Roukos, Salim  and
Ward, Todd  and
Zhu, Wei-Jing},
 booktitle = {Proceedings of the 40th Annual Meeting of the Association for Computational Linguistics},
 doi = {10.3115/1073083.1073135},
 editor = {Isabelle, Pierre  and
Charniak, Eugene  and
Lin, Dekang},
 pages = {311--318},
 publisher = {Association for Computational Linguistics},
 title = {{B}leu: a Method for Automatic Evaluation of Machine Translation},
 url = {https://aclanthology.org/P02-1040},
 year = {2002}
}

@inproceedings{lin2004rouge,
 address = {Barcelona, Spain},
 author = {Lin, Chin-Yew},
 booktitle = {Text Summarization Branches Out},
 pages = {74--81},
 publisher = {Association for Computational Linguistics},
 title = {{ROUGE}: A Package for Automatic Evaluation of Summaries},
 url = {https://aclanthology.org/W04-1013},
 year = {2004}
}

@inproceedings{zhangbertscore,
 author = {Tianyi Zhang and
Varsha Kishore and
Felix Wu and
Kilian Q. Weinberger and
Yoav Artzi},
 bibsource = {dblp computer science bibliography, https://dblp.org},
 booktitle = {8th International Conference on Learning Representations, {ICLR} 2020,
Addis Ababa, Ethiopia, April 26-30, 2020},
 publisher = {OpenReview.net},
 title = {BERTScore: Evaluating Text Generation with {BERT}},
 url = {https://openreview.net/forum?id=SkeHuCVFDr},
 year = {2020}
}

@inproceedings{yuan2021bartscore,
 author = {Weizhe Yuan and
Graham Neubig and
Pengfei Liu},
 bibsource = {dblp computer science bibliography, https://dblp.org},
 booktitle = {Advances in Neural Information Processing Systems 34: Annual Conference
on Neural Information Processing Systems 2021, NeurIPS 2021, December
6-14, 2021, virtual},
 editor = {Marc'Aurelio Ranzato and
Alina Beygelzimer and
Yann N. Dauphin and
Percy Liang and
Jennifer Wortman Vaughan},
 pages = {27263--27277},
 title = {BARTScore: Evaluating Generated Text as Text Generation},
 year = {2021}
}

@article{jaech2024openai,
  title={OpenAI o1 System Card},
  author={Jaech, Aaron and Kalai, Adam and Lerer, Adam and Richardson, Adam and El-Kishky, Ahmed and Low, Aiden and Helyar, Alec and Madry, Aleksander and Beutel, Alex and Carney, Alex and others},
  journal={arXiv preprint arXiv:2412.16720},
  year={2024}
}

@article{li2023prd,
 author = {Li, Ruosen and Patel, Teerth and Du, Xinya},
 journal = {ArXiv preprint},
 title = {Prd: Peer rank and discussion improve large language model based evaluations},
 url = {https://arxiv.org/abs/2307.02762},
 volume = {abs/2307.02762},
 year = {2023}
}

@inproceedings{wang2023chatgpt,
  title={Is ChatGPT a Good NLG Evaluator? A Preliminary Study},
  author={Wang, Jiaan and Liang, Yunlong and Meng, Fandong and Sun, Zengkui and Shi, Haoxiang and Li, Zhixu and Xu, Jinan and Qu, Jianfeng and Zhou, Jie},
  booktitle={Proceedings of the 4th New Frontiers in Summarization Workshop},
  pages={1--11},
  year={2023}
}

@article{chiang2023closer,
  title={A closer look into automatic evaluation using large language models},
  author={Chiang, Cheng-Han and Lee, Hung-yi},
  journal={arXiv preprint arXiv:2310.05657},
  year={2023}
}

@inproceedings{sun2022bertscore,
  title={BERTScore is Unfair: On Social Bias in Language Model-Based Metrics for Text Generation},
  author={Sun, Tianxiang and He, Junliang and Qiu, Xipeng and Huang, Xuan-Jing},
  booktitle={Proceedings of the 2022 Conference on Empirical Methods in Natural Language Processing},
  pages={3726--3739},
  year={2022}
}

@inproceedings{chen2024mllm,
  title={Mllm-as-a-judge: Assessing multimodal llm-as-a-judge with vision-language benchmark},
  author={Chen, Dongping and Chen, Ruoxi and Zhang, Shilin and Wang, Yaochen and Liu, Yinuo and Zhou, Huichi and Zhang, Qihui and Wan, Yao and Zhou, Pan and Sun, Lichao},
  booktitle={Forty-first International Conference on Machine Learning}
}

@article{TCP,
  title={Trajectory-guided control prediction for end-to-end autonomous driving: A simple yet strong baseline},
  author={Wu, Penghao and Jia, Xiaosong and Chen, Li and Yan, Junchi and Li, Hongyang and Qiao, Yu},
  journal={Advances in Neural Information Processing Systems},
  volume={35},
  pages={6119--6132},
  year={2022}
}

@inproceedings{MomAD,
  title={Don't Shake the Wheel: Momentum-Aware Planning in End-to-End Autonomous Driving},
  author={Song, Ziying and Jia, Caiyan and Liu, Lin and Pan, Hongyu and Zhang, Yongchang and Wang, Junming and Zhang, Xingyu and Xu, Shaoqing and Yang, Lei and Luo, Yadan},
  booktitle={Proceedings of the Computer Vision and Pattern Recognition Conference},
  pages={22432--22441},
  year={2025}
}

@article{zimmerlin2024hidden,
  title={Hidden biases of end-to-end driving datasets},
  author={Zimmerlin, Julian and Bei{\ss}wenger, Jens and Jaeger, Bernhard and Geiger, Andreas and Chitta, Kashyap},
  journal={arXiv preprint arXiv:2412.09602},
  year={2024}
}

@article{gao2024mini,
  title={Mini-internvl: a flexible-transfer pocket multi-modal model with 5\% parameters and 90\% performance},
  author={Gao, Zhangwei and Chen, Zhe and Cui, Erfei and Ren, Yiming and Wang, Weiyun and Zhu, Jinguo and Tian, Hao and Ye, Shenglong and He, Junjun and Zhu, Xizhou and others},
  journal={Visual Intelligence},
  volume={2},
  number={1},
  pages={32},
  year={2024},
  publisher={Springer}
}

@article{team2024qwen2,
  title={Qwen2 technical report},
  author={Team, Qwen and others},
  journal={arXiv preprint arXiv:2407.10671},
  volume={2},
  number={3},
  year={2024}
}

@article{loshchilov2017decoupled,
  title={Decoupled weight decay regularization},
  author={Loshchilov, Ilya and Hutter, Frank},
  journal={arXiv preprint arXiv:1711.05101},
  year={2017}
}

@inproceedings{rasley2020deepspeed,
  title={Deepspeed: System optimizations enable training deep learning models with over 100 billion parameters},
  author={Rasley, Jeff and Rajbhandari, Samyam and Ruwase, Olatunji and He, Yuxiong},
  booktitle={Proceedings of the 26th ACM SIGKDD international conference on knowledge discovery \& data mining},
  pages={3505--3506},
  year={2020}
}
\bibliographystyle{icml2026}

%%%%%%%%%%%%%%%%%%%%%%%%%%%%%%%%%%%%%%%%%%%%%%%%%%%%%%%%%%%%%%%%%%%%%%%%%%%%%%%
%%%%%%%%%%%%%%%%%%%%%%%%%%%%%%%%%%%%%%%%%%%%%%%%%%%%%%%%%%%%%%%%%%%%%%%%%%%%%%%
% APPENDIX
%%%%%%%%%%%%%%%%%%%%%%%%%%%%%%%%%%%%%%%%%%%%%%%%%%%%%%%%%%%%%%%%%%%%%%%%%%%%%%%
%%%%%%%%%%%%%%%%%%%%%%%%%%%%%%%%%%%%%%%%%%%%%%%%%%%%%%%%%%%%%%%%%%%%%%%%%%%%%%%
\newpage
\appendix
\onecolumn

\textcolor{blue}{\textit{CriticVLA.mp4}} demonstrates examples where CriticVLA successfully performs direction correction and speed correction. Meanwhile, it presents a comparative analysis of the performance among CriticVLA, Simlingo, and the Stage-1 model. For further details, refer to the attached video.

\section{Proof}
\label{sec:proof}
% \onecolumn

\subsection{The proof of Theorem 1}

% \begin{assumption}
\textbf{Assumption 3.1} (Critic Improvement)
The critic is assumed to have a guaranteed improvement ratio $\beta \in (0,1)$ on the training set, i.e., for all $A \in $ training set $T$,
\begin{equation}
Q(C(A)) - Q(A) \ge \beta \, [Q^*(V,L) - Q(A)].
\end{equation}
This means the critic can improve the value of any action by at least a fraction $\beta$ of its gap to the optimal value. A larger value of $\beta$ indicates a better ability of the critic.
% \end{assumption}

% \begin{assumption}
\textbf{Assumption 3.2} (Lipschitz Continuity)
There exists a constant $L_Q > 0$ such that for all $A^{(1)}, A^{(2)} \in \mathcal{A}$,
\begin{equation}
|Q(A^{(1)}) - Q(A^{(2)})| \le L_Q \, d(A^{(1)}, A^{(2)}),
\end{equation}
ensuring that small changes in actions lead to bounded changes in value. For critic function, there exists a constant $L_C > 0$ such that for all $A^{(1)}, A^{(2)} \in \mathcal{A}$,
\begin{equation}
d(C(A^{(1)}), C(A^{(2)})) \le L_C \, d(A^{(1)}, A^{(2)}),
\end{equation}
which indicates that the critic’s refinement of actions is smooth and stable.
% \end{assumption}

\begin{mdframed}[backgroundcolor=gray!8]
\begin{minipage}{\linewidth}
\begin{theorem}
For fixed visual input $V$ and language input $L$, let $A_0$ denote the initial action. Suppose the critic model satisfies Assumption \ref{ass:1} (critic improvement on the training set) and Assumption \ref{ass:2} (Lipschitz continuity). Let $\rho_0 = \rho(A_0)$ denote the distance from $A_0$ to the training set $T$. Then, the value of the improved action $A_1 = C(A_0)$ is lower bounded by
\begin{equation}
    Q(A_1) \ge \beta Q^* +  (1 - \beta) \, Q(A_0) - \rho_0 L_Q (1 - \beta + L_C),
    \label{eq:Qa1}
\end{equation}
where $Q^*$ is the upper bound value for given $V,L$ and the corresponding optimal action.
\end{theorem}
\end{minipage}
\end{mdframed}

\begin{proof}

    Since $\rho_0 = \min_{A \in T} d(A_0, A)$ and the training set $T$ contains finite samples, we can find $\hat{A} \in T$ such that
\begin{equation}
     \rho_0 = d(A_0,\hat{A}) = \min_{A \in T} d(A_0, A).
     \label{eq:rho_0}
\end{equation}
Next, by the Lipschitz continuity assumption of the critic function $C$ in Assumption \ref{ass:2}, we have
\begin{equation}
    d(C(A_0),C(\hat{A})) \leq L_C \cdot d(A_0,\hat{A}) = L_C \cdot \rho_0.
    \label{eq:lip1}
\end{equation}
Furthermore, from Eq.~(\ref{eq:lip1}) and Assumption \ref{ass:2} (which also states that $Q$ is $L_Q$-Lipschitz continuous), we obtain
\begin{equation}    
    \begin{aligned}
    |Q(C(A_0)) - Q(C(\hat{A}))| &\leq L_Q d(C(A_0),C(\hat{A})) \\
     &\leq L_Q \cdot L_C \cdot d(A_0,\hat{A}) & \hfill \text{Since Eq.~(\ref{eq:lip1})} \\
     & = L_Q \cdot L_C \cdot \rho_0,
\end{aligned}
\label{eq:qclip}
\end{equation}
Therefore, we have
\begin{equation}
    Q(C(A_0)) \geq Q(C(\hat{A})) - L_QL_C\rho_0,
    \label{eq:temp1}
\end{equation}
Subsequently, as $\hat{A} \in T$, the condition in Assumption \ref{ass:1} gives
\begin{equation}
   Q(C(\hat{A})) - Q(\hat{A}) \ge \beta \, [Q^*(V,L) - Q(\hat{A})].
   \label{eq:temp2}
\end{equation}
Moreover, we rewrite the inequality in Eq.~(\ref{eq:temp1}),
\begin{equation}    
    \begin{aligned}
    Q(C(A_0)) &\geq Q(C(\hat{A})) - L_QL_C\rho_0  & \hfill \text{Since  Eq.~(\ref{eq:temp1})} \\
     &\geq Q(\hat{A}) + \beta \, [Q^*(V,L) - Q(\hat{A})] - L_QL_C\rho_0  & \hfill \text{Since Eq.~(\ref{eq:temp2})} \\
     & = (1-\beta) Q(\hat{A}) + \beta Q^* - L_QL_C\rho_0.
\end{aligned}
\label{eq:qclip}
\end{equation}
where $Q^* = Q^*(V,L)$. Note that $A_1 = C(A_0)$, we further have
\begin{equation}    
    \begin{aligned}
    Q(A_1) &= Q(C(A_0)) \\
     &\geq (1-\beta) Q(\hat{A}) + \beta Q^* - L_QL_C\rho_0  & \hfill \text{Since Eq.~(\ref{eq:qclip})} \\
     & \geq (1-\beta) \cdot [Q(A_0) - L_Q\rho_0 ] + \beta Q^* - L_QL_C\rho_0 & \hfill \text{Assumption \ref{ass:2}}\\
     & = \beta Q^* + \underbrace{\colorbox[rgb]{0.854,0.910,0.988}{$(1 - \beta) \, Q(A_0)$}}_{\text{For Stage-1  }} - \underbrace{\colorbox[rgb]{1,0.8,0.8}{$\rho_0 L_Q (1 - \beta + L_C)$}}_{\text{For Stage-2  }}.
\end{aligned}
\label{eq:qclip1}
\end{equation}

    \qedhere
\end{proof}

\subsection{The proof of Theorem 2}

\begin{mdframed}[backgroundcolor=gray!8]
\begin{minipage}{\linewidth}
\begin{theorem}
For the initial action $A_0$ under given input $V$ and $L$ in Theorem \ref{theo:1}, Let $A_k = C(A_{k-1}), k \geq 1$, and $\rho_k = \rho(A_k)$ denote the distance from $A_k$ to the training set $T$. A sufficient condition for the sequence $\{Q(A_k)\}$ to be monotonically increasing is that
\begin{equation}
  \rho_k \leq c_\beta\cdot [Q^* - Q(A_k)],
  % \label{eq:Qa2}
\end{equation}
where $Q^*$ is the upper bound value for given $V,L$ and the corresponding optimal action. $c_\beta$ is the constant related to $\beta$ and $c_\beta>0$.
\end{theorem}
\end{minipage}
\end{mdframed}

\begin{proof}

Since $\rho_k = \rho(A_k)$, using Theorem 1, we can derive the following lower bound estimate for the $k$-th improvement:
\begin{equation}
    Q(A_{k+1}) \ge \beta Q^* +  (1 - \beta) Q(A_k) - \rho_k L_Q (1 - \beta + L_C),
    \label{eq:temp3}
\end{equation}
where $k\geq 0$. To require the sequence $\{Q(A_k)\}$ to be monotonically increasing means that for all $k$,
\begin{equation}
    Q(A_{k} ) \leq Q(A_{k+1}).
\end{equation}
According to Eq.~(\ref{eq:temp3}), we only need to satisfy a sufficient condition:
\begin{equation}
    Q(A_{k+1}) \ge \beta Q^* +  (1 - \beta) Q(A_k) - \rho_k L_Q (1 - \beta + L_C) \geq Q(A_{k}).
    \label{eq:qak}
\end{equation}
Thus, we have
\begin{equation}    
    \begin{aligned}
    \beta Q^* +  (1 - \beta) Q(A_k) - \rho_k L_Q (1 - \beta + L_C) &\geq Q(A_{k}) & \hfill \text{Since Eq.~(\ref{eq:qak})}\\
    \beta Q^* +  (1 - \beta -1) Q(A_k) &\geq \rho_k L_Q (1 - \beta + L_C) \\
    \beta [Q^* - Q(A_k)] & \geq \rho_k L_Q (1 - \beta + L_C).
\end{aligned}
\label{eq:qclip1}
\end{equation}
Let $c_\beta = \beta/[L_Q(1-\beta+L_C)]$; then we have
\begin{equation}
    \rho_k \leq \frac{\beta}{L_Q(1-\beta+L_C)}[Q^* - Q(A_k)] \equiv c_\beta [Q^* - Q(A_k)].
\end{equation}
Here, since $\beta \in (0,1)$ and the Lipschitz constants $L_Q$ and $L_C$ are both positive, it follows that $c_\beta>0$.

    \qedhere
\end{proof}

% \twocolumn
\subsection{The verification of the assumptions}

As mentioned in the main text, Assumption \ref{ass:1} is established on the training set $T$. It can be reasonably guaranteed as long as the critic model is sufficiently well-trained. Moreover, we do not require the critic to achieve perfect performance (i.e., $\beta \to 1$), instead, $\beta \in (0,1)$ is sufficient, indicating any positive improvement over the initial policy. Hence, this assumption imposes only a moderately weak requirement. Therefore, Assumption \ref{ass:1} is a very mild assumption that does not require numerical verification.

And, the focus of numerical validation lies on Assumption \ref{ass:2}. This section mainly conducts numerical verification for Assumption \ref{ass:2} in the autonomous driving scenario.
Specifically, we sample 20,000 trajectories from 12.9 million annotated trajectories for analysis. First, regarding the part of Assumption \ref{ass:2} related to the $Q$ function: for an action $A$, $Q(A)$ can be defined as the proportion of untriggered risks. 
If $A$ triggers no risks, then $Q(A)=1$. In the experimental setting of this paper, $d(\cdot,\cdot)$ adopts the Euclidean norm.

\begin{figure}[h]
  \centering
 \includegraphics[width=0.9\linewidth]{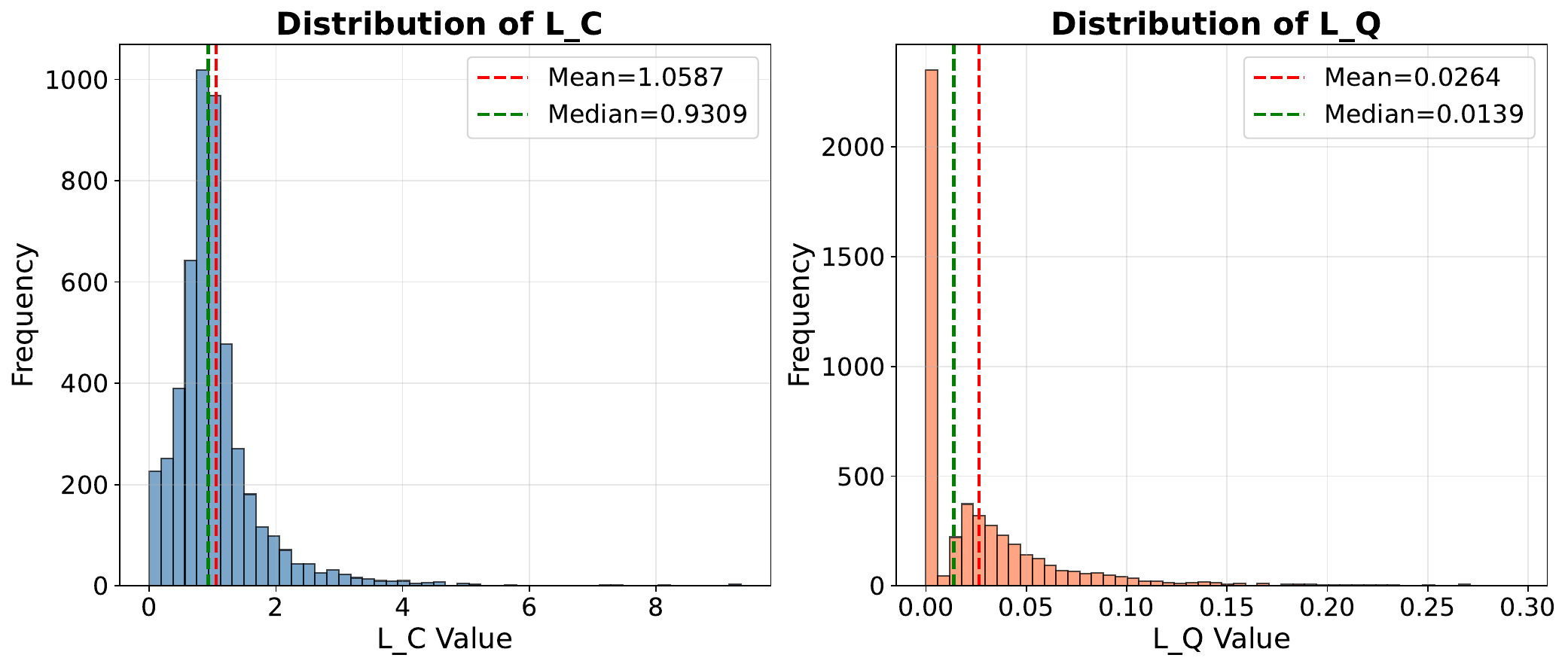}
  \caption{The distribution of Lipschitz constant $L_Q$ and $L_C$ in different scenarios.}
\label{fig:distri}
\end{figure}

As shown in Fig.~(\ref{fig:distri}), the distributions of $L_Q$ and $L_C$ from the sampled data exhibit the following characteristics: they are concentrated near the origin, presenting long-tailed distributions that first rise and then decline. Additionally, both values have obvious upper bounds, meaning it is easy to construct $L_Q$ and $L_C$ that satisfy the inequality requirement of Assumption \ref{ass:2}. In other words, Assumption \ref{ass:2} holds in the autonomous driving scenario involved in this paper.

\section{Model Architecture and Implementation Details}
\label{sec:detail}

\subsection{Overview and Terminology Explanation}
This section details the architecture, training methodology, hyperparameters, and input/output specifications of our proposed CriticVLA. The framework operates sequentially: first generating an initial trajectory, then performing a natural language-based risk assessment, and finally refining the trajectory based on this semantic analysis.

The model receives the following inputs:
\begin{itemize}
    \item Visual Input ($V$): A single front-facing camera image $I \in \mathbb{R}^{H \times W \times 3}$.
    \item State Input ($\mathcal{S}$): The ego vehicle's current speed $v$ (scalar).
    \item Navigation Input ($\mathcal{L}$): A 2D GPS target point defining the long-range driving goal relative to the ego vehicle.
\end{itemize}

\subsubsection{Route Waypoints and Speed Waypoints}
To achieve robust and precise autonomous navigation, our model predicts the future motion of the ego-vehicle using a dual-trajectory representation. This scheme decomposes the planning task into two complementary components: \textit{Geometric Route Waypoints} and \textit{Temporal Speed Waypoints}. This separation allows for the decoupling of lateral and longitudinal control, ensuring both geometric fidelity and dynamic feasibility.

\paragraph{Geometric Route Waypoints.}
We define the geometric path as a sequence of spatial coordinates $\mathbf{P} \in \mathbb{R}^{N_{p} \times 2}$, where $N_{p}$ represents the number of route waypoints. These points are generated with a fixed spatial interval (e.g., 1 meter between consecutive points), independent of the time required to traverse them. 
\begin{itemize}
    \item \textit{Role:} This representation captures the pure geometry of the intended driving path, focusing on curvature, lane alignment, and spatial obstacle avoidance.
    \item \textit{Utility:} By maintaining a high spatial density, $\mathbf{P}$ provides a smooth and continuous reference line for the lateral controller (e.g., a PID steering controller), minimizing cross-track error and ensuring stable steering maneuvers regardless of the vehicle's speed.
\end{itemize}
It should be noted that route waypoints may also be abbreviated as \textit{route} or \textit{path} in the text.

\paragraph{Temporal Speed Waypoints.}
Simultaneously, the model predicts a sequence of temporal speed waypoints $\mathbf{W} \in \mathbb{R}^{N_{w} \times 2}$, where $N_{w}$ denotes the number of future states. Unlike the geometric path, these points are generated at fixed time intervals (e.g., $\Delta t = 0.25$ seconds).
\begin{itemize}
    \item \textit{Role:} This representation implicitly encodes the vehicle's longitudinal velocity profile and acceleration dynamics. The Euclidean distance between consecutive temporal points reflects the intended speed at that specific moment.
    \item \textit{Utility:} $\mathbf{W}$ serves as the reference for the longitudinal controller. It dictates the throttle and brake commands required to manage speed transitions, such as decelerating before a sharp turn or accelerating during a lane change.
\end{itemize}

\subsection{Stage-1: Initial Trajectory Generation}
\paragraph{Input Encoding.} 
We adopt the InternVL2-1B model from the Mini-InternVL family \cite{gao2024mini} as our primary VLA backbone. This model consists of the InternViT-300M-448px Vision Encoder and the Qwen2-0.5B-Instruct \cite{team2024qwen2} Language Model (LLM). The multi-view RGB images are encoded by the Vision Encoder, providing latent visual features. The navigation input ($\mathcal{L}$, the GPS target point) is processed by a dedicated Waypoint Input Adaptor, a 3-layer MLP with ReLU activations, to convert the numerical coordinate inputs into token embeddings of dimension $D_{model}=2048$. These embeddings are then concatenated with specific navigation prompt tokens (\textit{e.g.}, \texttt{<IMG\_CONTEXT>}) for integration by the LLM.

\paragraph{Output Heads.} 
The LLM processes the concatenated visual and language/navigation features. Its output hidden states are projected into the trajectory space via the Driving Adaptor, which employs a DETR-based prediction approach \cite{carion2020end} using learnable queries. We utilize two types of learnable queries within the LLM's structure. First are 512 vision queries, which are used for latent multimodal feature extraction and grounding within the LLM's attention mechanism. Second are 16 trajectory queries, which interact with the contextualized hidden states and are directly mapped to the output trajectory $A_0$ by the output heads.

The Driving Adaptor comprises specialized heads acting on the hidden states corresponding to the Trajectory Queries. (1) The Route Waypoints Head. It predicts $N_p=20$ route waypoints. It uses an MLP (Linear $\rightarrow$ SiLU $\rightarrow$ Linear) to output cumulative deltas from the current ego-position, which are then summed to form the absolute coordinates of $\mathbf{P} \in \mathbb{R}^{20 \times 2}$. (2) The Speed Waypoints Head: Predicts $N_w=10$ speed waypoints, which uses a similar MLP structure to predict the coordinates of $\mathbf{W} \in \mathbb{R}^{10 \times 2}$ at fixed time intervals. (3) The Waypoint Input Adaptor,  which encodes numerical coordinate inputs into token embeddings via a 3-layer MLP with ReLU activations. The resulting output $A_0$ is the rough route tensor $\mathbf{P}$ and the speed profile $\mathbf{W}$, serving as the geometric foundation for the second stage. Prompt details for Stage-1 are provided in Section \ref{sec:prompt}.

\paragraph{Training data.}
To counter the overabundance of easy or uneventful data in driving logs, we implement a Weighted Bucket Sampling strategy following\cite{renz2025simlingo}. The full training data is partitioned into specific buckets corresponding to critical driving situations (\textit{e.g.}, complex intersections, pedestrian crossings, lane changes). During each training epoch, we sample from these buckets with non-uniform probabilities that favor challenging scenarios. This significantly concentrates the training effort on critical decision-making, while still including a sufficient ratio of non-critical data. This approach results in a controlled epoch size of approximately 480,000 samples.

\subsection{Stage-2: Trajectory Refinement with Language Reasoning}
The architecture of the refinement stage is designed to integrate geometric priors from the rough trajectory with high-level semantic reasoning. We reuse the InternVL2-1B backbone, treating it as a multimodal critic to evaluate and optimize the rough trajectory $A_0$ from Stage-1. This stage introduces three primary components: geometric input encoders, the language critic head, and the query-based delta adaptors. We enable Low-Rank Adaptation (LoRA) \cite{hu2022lora} on the backbone (rank $r=32$, $\alpha=64$, dropout $0.1$) targeting all linear layers to make the model trainable for reasoning tasks while maintaining efficiency.

\paragraph{Geometric Input Encoders.}
To ingest the coarse trajectory $A_0$ generated in Stage-1 into the LLM's token space, we use specialized encoders that map the geometric coordinates into the model's embedding space ($d_{model}=512$).
\begin{itemize}
    \item The \textbf{Delta Waypoint Encoder} processes the rough route waypoints ($P_{rough} \in \mathbb{R}^{20 \times 2}$).
    \item The \textbf{Delta Speed Encoder} processes the rough speed waypoints ($V_{rough} \in \mathbb{R}^{10 \times 2}$).
\end{itemize}
Both encoders share a Multi-Layer Perceptron (MLP) architecture: Linear($2 \rightarrow 256$) $\rightarrow$ ReLU $\rightarrow$ Linear($256 \rightarrow 512$) $\rightarrow$ ReLU $\rightarrow$ Linear($512 \rightarrow d_{model}$). These projected embeddings replace placeholder tokens (\textit{e.g.}, \texttt{<wpt>... </wpt>}) within the Delta Prompt, allowing the LLM to contextually reason over the initial plan.

\paragraph{Language Head.}
Parallel to the trajectory heads, the \textit{Delta Language Adaptor} utilizes the standard language modeling head of the backbone. It is responsible for autoregressively generating the risk analysis tokens and action recommendations, which subsequently serve as context for the Delta Adaptor in the trajectory refinement phase.  The analysis includes a structured risk assessment dictionary, evaluating six categories (Collision, Speed, Direction, Pedestrian, Stop Sign, Traffic Light), followed by concrete action recommendations (\textit{e.g.}, ``reduce speed from 7.4 m/s to 3.5 m/s''). This helps to inject semantic reasoning to refine the rough trajectory.

\paragraph{Query-Based Delta Adaptors.}
The final refinement is driven by a set of Learnable Refinement Queries ($Q_{route} \in \mathbb{R}^{1 \times 20 \times d_{model}}$ and $Q_{speed} \in \mathbb{R}^{1 \times 10 \times d_{model}}$). These queries anchor the model's attention to the trajectory refinement task. The output hidden states corresponding to these queries are fed into the DeltaAdaptor. This adaptor predicts the refinement deltas ($\Delta \mathbf{P}, \Delta \mathbf{W}$) rather than absolute coordinates. It employs a dedicated MLP structure: Linear($d_{model} \rightarrow 2d_{mlp}$) $\rightarrow$ SiLU $\rightarrow$ Linear($2d_{mlp} \rightarrow d_{mlp}$) $\rightarrow$ SiLU $\rightarrow$ Linear($d_{mlp} \rightarrow 2$), where $d_{mlp}=256$.

\paragraph{Two-Phase Inference and Output.} 
Stage-2 inference executes in a strict two-phase forward pass driven by the Delta Prompt (which includes a zero-shot example to guide the reasoning format):
\begin{enumerate}
    \item \textbf{Semantic Analysis Phase:} The model first performs an autoregressive decoding pass using the Language Critic Head, generating the complete risk assessment and action recommendations.
    \item \textbf{Refinement Phase:} Using the generated analysis as context, the model employs the Query-based Delta Adapters to directly predict refined trajectories.
\end{enumerate}

\paragraph{Training data.}
The CriticVLA refinement stage is trained exclusively on the CriticDrive dataset, which totals 3,081,848 annotated driving frames. This dataset is logically partitioned into three components: Ground Truth (GT) Data (re-formatted from the stage1 training set), the Model-Generated Subset (MGS) (2,023,499 frames), and the Extra Perturbation-Augmented Subset (EPAS) (1,058,349 frames). Details of this dataset can be found in Section ~\ref{sec:CriticDriveComposition} We evaluate performance under two primary training regimes:
\begin{itemize}
    
\item Base CriticDrive: This regime focuses on training the critic primarily on our generated error samples (MGS). Crucially, the MGS is augmented with 15\% Ground Truth (GT) Data to teach the critic the necessary skill of identifying already-optimal trajectories where the required refinement delta is zero. This configuration yields a controlled epoch size of 480,000 samples.

\item Full CriticDrive (with Augmentation): This regime incorporates the complete set of error samples by including the Extra Perturbation-Augmented Subset (EPAS). The EPAS is included to maximize diversity in failure patterns and enhance the model's refinement capabilities, aligning with the necessity of reducing the distance to the training set ($\rho_0$) as characterized by Theorem 1. Training samples are drawn with a $1:1$ ratio from the combined (MGS + GT Data) pool and the EPAS pool. This results in a full epoch size of $960,000$ samples.

\end{itemize}

% \subsubsection{Introduction to Training Hyperparameters (Stage-2)}
% \subsubsection{Model Architecture Details (Stage-2)}
\subsection{Model Inference Details}
We evaluate the driving capability of CriticVLA on Bench2Drive\cite{jia2024bench2drive} closed-loop benchmark with the CARLA simulator version 0.9.15 \cite{dosovitskiy2017carla}. The evaluation follows a multi-stage pipeline: 

\paragraph{Input Preparation} Multi-view RGB camera images are captured and preprocessed via dynamic resizing. Meanwhile, GPS and IMU sensors provide vehicle localization and orientation data. A route planner generates target points based on the global route plan, which are transformed into ego-centric coordinates along with the current vehicle speed to form the input feature set.

\paragraph{Model Inference} CriticVLA processes multi-modal inputs through a two-stage refinement mechanism. In the first stage, conditioned on visual observations, target points, and language inputs, the model generates initial trajectory waypoints (route waypoints and speed waypoints). Subsequently, a language-guided refinement stage employs greedy sampling to generate natural language descriptions of initial trajectory risk analysis and further action suggestions, which are utilized to refine the initial predictions into better trajectory predictions (route waypoints and speed waypoints).

\paragraph{Vehicle control} The predicted waypoints are interpolated using piecewise cubic Hermite interpolation to ensure spatial smoothness. A dual PID control framework converts the refined trajectory into low-level vehicle control commands: a lateral PID controller calculates steering angles by tracking interpolated route waypoints, while a longitudinal PID controller regulates throttle and brake based on the desired speed derived from speed waypoints. Braking logic is triggered when the desired speed falls below a threshold or the speed ratio exceeds a safety limit. The generated control signals are executed in the CARLA simulation environment to evaluate the autonomous driving performance of the agent.
% We obtain speed from speed waypoints and angle from route waypoints, then we use two PID controllers to get the steering and angle acceleration. 

\subsection{Format of Prompts and Answers}
\label{sec:prompt}
For detailed specifications of the reasoning mechanism, the templates for the prompts and generated answers used are illustrated in Figure~\ref{fig:prompt}.

\begin{figure}%[t]
    \centering
    \includegraphics[width=0.99\linewidth]{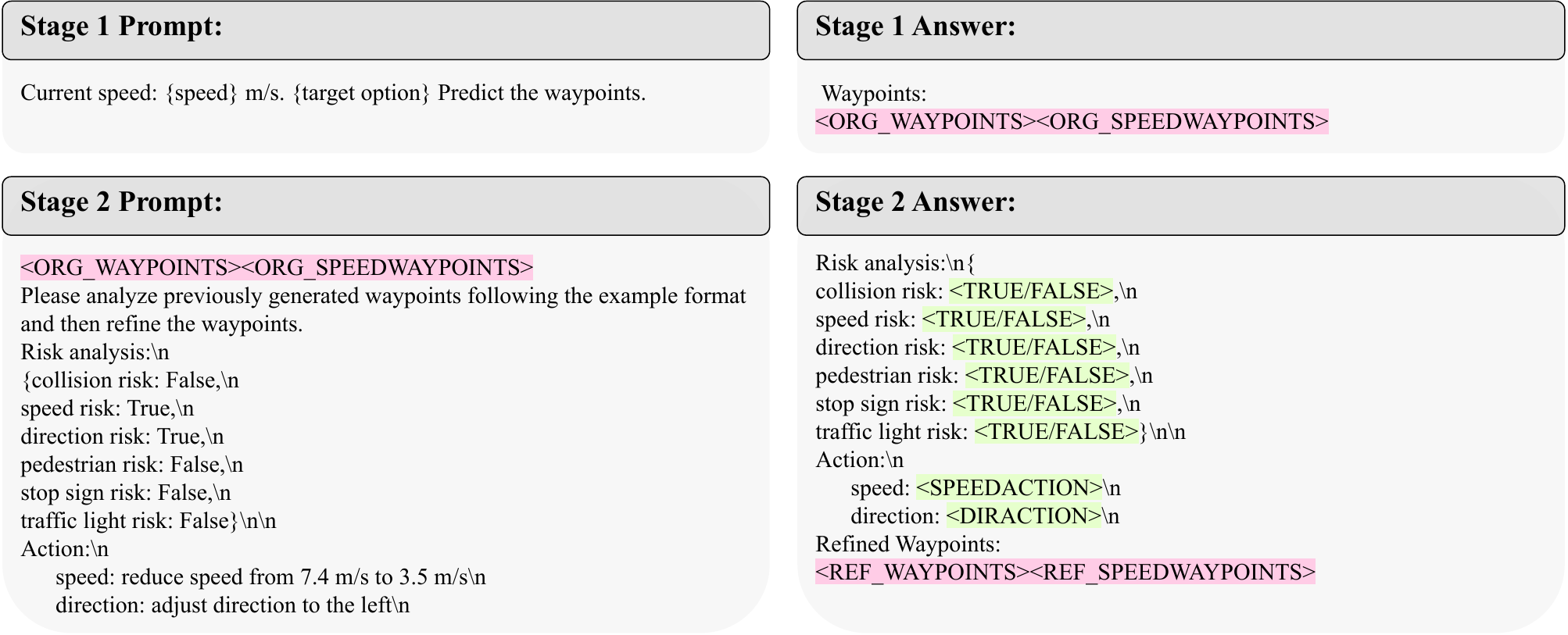}
    % \vspace{-0.3cm}
    \caption{The template of prompt and answer. }
    \label{fig:prompt}
\end{figure}

\subsection{Implementation Details}
\paragraph{Hyperparameters and Optimization.} 
The model is trained on 8 NVIDIA A100 GPUs using DeepSpeed ZeRO Stage-2~\cite{rasley2020deepspeed} with 16-bit mixed precision. We use the AdamW optimizer~\cite{loshchilov2017decoupled} with a learning rate of $3 \times 10^{-5}$, weight decay of $0.1$, and $\beta=(0.9, 0.999)$. The learning rate follows a OneCycleLR schedule with a 5\% warmup period. Gradient clipping is set to 0.3 to prevent divergence.

\paragraph{Data Augmentation}
To enhance robustness, we apply random image augmentations (probability $0.5$) and image shift augmentations (probability $0.5$) during training. The trajectory prediction horizon is set to $T=11$ timesteps (future 2.5s at 4Hz), predicting 20 spatial waypoints.

\paragraph{Training} We leverage DeepSpeed v2~\cite{rasley2020deepspeed} for efficiency and memory management. The fine-tuning of CriticVLA consists of two stages, as detailed below.

\noindent\textit{Stage-1 Fine-tuning.} For the Stage-1 model, we employ a mixed fine-tuning strategy. The vision encoder and all projection layers are fully fine-tuned to maximize their ability to extract driving-relevant visual features. Conversely, the Language Model is fine-tuned using LoRA \cite{hu2022lora} applied to all its linear layers. The LoRA rank dimension is set to $r=32$ and the scaling factor $\alpha$ is 64. Stage-1 is trained for a total of 13 epochs.

\noindent\textit{Stage-2 Fine-tuning.} In the critic stage, we reuse the same InternVL2-1B backbone initialized with the weights from the trained Stage-1 model. Only the parameters related to the Critic Param (as shown in Figure 1(b)) are fine-tuned, while the parameters of the Stage-1 Generator (including the Vision Encoder and LLM) are frozen. The model is trained on the synthetic CriticDrive dataset for 13 epochs. This stage adopts the same LoRA configuration ($r=32, \alpha=64$) and optimization hyperparameters as Stage-1 to ensure consistency and facilitate efficient knowledge transfer.

\paragraph{Loss Function}
The total loss $\mathcal{L}_{total}$ is a weighted sum of:
\begin{equation}
    \mathcal{L}_{total} = \lambda_{lang}\mathcal{L}_{CE} + \lambda_{route}\mathcal{L}_{Smooth-L1}^{route-wps} + \lambda_{speed}\mathcal{L}_{Smooth-L1}^{speed-wps}
\end{equation}
where $\mathcal{L}_{CE}$ is the cross-entropy loss for language generation, and $\mathcal{L}_{L1}$ denotes the smooth-L1 distance between the refined predictions and ground truth waypoints.

\section{Introduction to CriticDrive}
\label{sec:CriticDrive}

\subsection{Construction Method of the CriticDrive Dataset}
\label{sec:CriticDrive_construct}
To enable the training of CriticVLA, we formulate a comprehensive \textbf{\textit{Critic Reasoning Creator}} to automatically construct CriticDrive dataset. This rule-based heuristic engine evaluates predicted trajectory against expert ground truth (GT) demonstrations, environmental constraints, and dynamic actor behaviors. Our construction method decomposes the reasoning process into two stages: \textbf{Risk Analysis}, which identifies potential hazards across four domains, and \textbf{Action Suggestion}, which synthesizes corrective natural language directives.

\subsubsection{Risk Analysis Generation}
The risk analysis generator performs a spatiotemporal evaluation of the predicted trajectories. We define the risk set $\mathcal{R}$ across four dimensions: Lateral Deviation, Longitudinal Kinematics, Collision Safety, and Contextual Complexity.

\paragraph{\textit{Lateral Risk Analysis}} \mbox{} \\

\vspace{-10pt}
\noindent We evaluate the spatial alignment of the predicted route $\mathcal{P}_{pred}$ against the expert ground truth route $\mathcal{P}_{gt}$.
\begin{itemize}
    \item \textbf{Angular Deviation:} We calculate the heading vectors by averaging the coordinates of the final $k=5$ waypoints for both $\mathcal{P}_{pred}$ and $\mathcal{P}_{gt}$. A directional risk is flagged if the angular difference $\Delta \theta$ exceeds the threshold $\tau_{\theta} = 7.5^{\circ}$.
    \item \textbf{Route Offset (Cross-Track Error):} We compute the Cross-Track Error (CTE) of the predicted path relative to the expert trajectory. A \textit{Route Offset} risk is identified if the maximum lateral displacement exceeds $\tau_{lat} = 2.0$ meters.
    \item \textbf{Topological Mismatch:} We classify the route geometry (e.g., Straight, Left/Right Turn, Lane Change) based on cumulative curvature and lateral displacement. A risk is recorded if the predicted topology class differs from the expert class (e.g., predicting ``Straight" when the expert executes a ``Lane Change").
\end{itemize}

\paragraph{\textit{Longitudinal Risk Analysis}} \mbox{} \\

\vspace{-10pt}
\noindent We calculate predicted velocity profile $V_{pred}$ and the ground truth profile $V_{gt}$  from predicted speed waypoints $W_{pred}$ and expert speed waypoints $W_{gt}$. Then, we assess the safety and intent consistency of $V_{pred}$ against  $V_{gt}$ and the legal speed limit $V_{limit}$.
\begin{itemize}
    \item \textbf{Speed Limit Compliance:} A critical risk is flagged if the current speed exceeds a safety margin of the legal limit: $v_{curr} > 0.9 \times V_{limit}$.
    \item \textbf{Speed Deviation:} We compare the predicted average speed (first 3 frames) and future speed (final frame) against expert data. A speed risk is identified if both the relative error exceeds $\tau_{rel} = 20\%$ and the absolute difference exceeds $\tau_{abs} = 0.5$ m/s.
    \item \textbf{Intent Classification:} We classify the velocity trend into four categories: \textit{Accelerating}, \textit{Decelerating}, \textit{Maintaining}, or \textit{Braking-to-Stop}. This is achieved using a time-weighted gradient of the future waypoints. A semantic risk arises if the predicted intent contradicts the expert intent (e.g., agent predicts ``Maintaining" while the expert is ``Braking-to-Stop").
\end{itemize}

\paragraph{\textit{Collision Verification}} \mbox{} \\

\vspace{-10pt}
\noindent To detect imminent hazards, we perform a rigorous intersection test. We project the ego-vehicle's state into the future using a Kinematic Bicycle Model to generate a sequence of 3D Oriented Bounding Boxes (OBBs).
\begin{itemize}
    \item \textbf{Mechanism:} For each timestep $t$ in the forecast horizon, we check for geometric overlap between the ego OBB and the OBBs of all surrounding actors (vehicles, pedestrians).
    \item \textbf{Algorithm:} We utilize the Separating Axis Theorem (SAT) to determine intersection. If no separating plane exists between the ego and an actor, a collision risk is confirmed, and the specific actor ID and class are recorded.
\end{itemize}

\paragraph{\textit{Environmental and Contextual Analysis}} \mbox{} \\

\vspace{-10pt}
\noindent We analyze the static environment and scene complexity to identify latent risks:
\begin{itemize}
    \item \textbf{Scene Complexity:} A scenario is deemed ``Complex" if the number of dynamic participants exceeds $N=6$.
    \item \textbf{Adverse Conditions:} We flag risks for poor visibility (rain, fog, night) or reduced traction (road wetness $> 40\%$).
    \item \textbf{Vulnerable Road Users (VRU):} A specific risk is triggered if a pedestrian is detected within a radial distance of $distance < 10$ meters.
    \item \textbf{Traffic Controls:} Risks are flagged for active Stop Signs, red traffic lights, or traffic rule violations (e.g., entering a forbidden lane).
\end{itemize}

\subsubsection{Action Suggestion Generation}
Based on the aggregated risk profile, our method synthesizes corrective action suggestions. This process employs a hierarchical logic to prioritize safety-critical maneuvers over comfort-related adjustments.

\paragraph{$\diamond$ Corrective Longitudinal Action} \mbox{} \\

\vspace{-10pt}
\noindent We generate speed commands by comparing the predicted speed $v_{pred}$ with the expert reference $v_{gt}$ when a longitudinal risk is present:
\begin{equation}
    A_{speed} = 
    \begin{cases} 
    \text{``Reduce speed from } v_{pred} \text{ to } v_{gt}\text{''} & \text{if } v_{pred} > v_{gt} \\
    \text{``Increase speed from } v_{pred} \text{ to } v_{gt}\text{''} & \text{if } v_{pred} < v_{gt} \\
    \text{``Maintain speed at } v_{pred}\text{''} & \text{otherwise}
    \end{cases}
\end{equation}

\paragraph{\textit{Corrective Lateral Action}} \mbox{} \\

\vspace{-10pt}
\noindent Directional suggestions are derived to negate the observed lateral offset:
\begin{itemize}
    \item If a \textit{Direction Risk} is detected, we determine the offset direction $D_{offset} \in \{\text{Left}, \text{Right}\}$. The system suggests the inverse maneuver: \textit{``Adjust direction to the [Opposite of $D_{offset}$]"} (e.g., if offset is Left, suggest Right).
    \item \textbf{Collision Override:} If a \textit{Collision Risk} is identified, the standard directional suggestion is overridden by a high-priority safety directive: \textit{``Collision risk with [Actor Class], proceed with caution and yield."}
    \item If no deviations are found, the system suggests \textit{``Maintain direction."}
\end{itemize}

\subsection{Composition of CriticDrive}
\label{sec:CriticDriveComposition}
To facilitate robust trajectory critique and refinement, we constructed a large-scale dataset comprising a total of \textbf{3,081,848} annotated driving frames. This dataset is structured into two distinct subsets: \textit{Model-Generated Trajectories} and \textit{Extra Heuristically Perturbed Trajectories}.

\subsubsection{Model-Generated Subset}
The first subset consists of \textbf{2,023,499} frames. For each frame, we employed the Stage-1 Model to generate a corresponding coarse trajectory. These predictions were then aligned with their respective expert ground-truth trajectories to generate dense critic annotations—specifically, risk analysis and action suggestions—following the methodology described in Section~\ref{sec:CriticDrive_construct}. 

Table~\ref{tab:rough} details the distribution of detected risks within this subset. It is crucial to note that the risk assessment is a multi-label classification problem; distinct risk categories frequently co-occur within a single scenario (e.g., a \textit{Pedestrian Risk} may exist simultaneously with a \textit{Traffic Light Risk}, or a \textit{Collision Risk} may accompany a \textit{Speed Risk}). Consequently, the statistics in Table~\ref{tab:rough} represent the frequency of positive instances for each specific risk category relative to the total number of samples.

\begin{table}[h]
\centering
\caption{\textbf{Distribution of Identified Risk Categories in the Model-Generated Subset.} 
    This table details the prevalence of specific risk types detected within the 2,023,499 coarse trajectories generated by the Stage-1 Model. Values represent the percentage of samples flagged with each risk. Note that the risk categories are \textit{non-exclusive}, meaning a single trajectory may exhibit multiple concurrent risks (e.g., both Collision Risk and Speed Risk).}
\vspace{-6pt}
\scalebox{1.0}{
% \begin{tabular}{>{\kern-\tabcolsep}*{8}{>$c<$}<{\kern-\tabcolsep}}
\begin{tabular}{@{}ccccccc@{}}
\toprule
\multirow{2}{*}{\textbf{Percentage}} & Collision & Speed & Direction & Pedestrian & Traffic Light & Stop Sign  \cr 
\cmidrule(lr){2-7}
& 13.22\% & 21.29\% & 36.98\% & 1.70\% & 11.81\% & 6.99\% \\
\bottomrule
\end{tabular}
}
\label{tab:rough}
\end{table}

\subsubsection{Extra Perturbation-Augmented Subset}
To further expand the diversity of failure patterns and enhance the model's refinement capability across varying degrees of trajectory quality, we introduced a second subset derived from \textbf{1,058,349} frames. Rather than relying solely on model outputs, we applied systematic perturbations to the expert trajectories, synthesizing a massive set of \textbf{10,895,598} extra rough trajectories.

Inspired by ~\cite{renz2025simlingo}, this heuristic perturbation strategy corrupts expert demonstrations to generate plausible yet risky trajectories. We categorize these perturbations into three distinct types: longitudinal velocity modifications, lateral lane deviations, and forced collision synthesis. All perturbed trajectories are dynamically verified using a kinematic bicycle model to ensure physical feasibility.

\paragraph{Longitudinal Velocity Modification}
We generate alternative speed profiles while maintaining the original geometric route to simulate behaviors such as over-speeding, unnecessary braking, or failure to adhere to target speeds. The perturbed velocity profiles are synthesized using the profile scaling mechanisms:

We apply a scalar factor $\gamma$ to the expert's speed profile $V_{exp}$. To simulate aggressive driving, we sample $\gamma \sim \mathcal{U}(1.1, 1.5)$. Conversely, to simulate overly cautious driving, we sample $\gamma \sim \mathcal{U}(0.3, 0.9)$.

\paragraph{Forced Lane Deviation}
To simulate unsafe lateral maneuvers, such as illegal lane changes or driving on sidewalks, we geometrically transform the expert route. The system identifies all adjacent zones, including same-direction lanes, opposite-direction lanes (risking head-on collisions), parking lanes, and sidewalks.

For a selected target zone with a lateral offset $W$ relative to the ego-lane, we synthesize a shifted trajectory $\mathcal{P}'$. Let $\mathcal{P} = \{p_1, \dots, p_T\}$ be the original waypoints. We define a start index $k_{start}$ and a transition length $L_{trans}$. The perturbed waypoints are computed by projecting the original points along the local normal vector $\vec{n}_i$:
\begin{equation}
    p'_i = p_i + \alpha_i \cdot W \cdot \vec{n}_i, \quad \text{where } \alpha_i \in [0, 1]
\end{equation}
Here, $\alpha_i$ represents the transition interpolation factor, shifting linearly from $0$ to $1$ over the transition duration. This geometric transformation is subsequently tracked by the bicycle model to generate the final physically grounded trajectory, simulating both successful and prohibited lane changes (e.g., crossing solid lines or entering oncoming traffic).

\paragraph{Forced Collision Synthesis}
We explicitly construct trajectories that result in collisions with static obstacles or dynamic actors to train the model's collision avoidance reasoning. The process proceeds in three steps:

\begin{enumerate}
    \item \textbf{Candidate Selection:} We identify candidate objects $O$ (vehicles, pedestrians, static meshes) within the ego-vehicle's future field of view. For dynamic actors, we utilize their future bounding boxes derived from the simulator's ground truth.
    \item \textbf{Spatiotemporal Matching:} For a selected object $O$ at a future timestep $t_{crash}$, we calculate the precise position $p_{obj}$. We then derive a constant collision speed $v_{crash}$ required for the ego-vehicle to reach $p_{obj}$ exactly at $t_{crash}$:
    \begin{equation}
        v_{crash} = \frac{\| p_{ego} - p_{obj} \| - \delta_{safety}}{\Delta t_{crash}}
    \end{equation}
    where $\delta_{safety}$ accounts for the physical extent of the vehicle and the object.
    \item \textbf{Route Interpolation:} We modify the geometric route to intersect with the target. The trajectory is reconstructed by interpolating between the current ego position and the collision point $p_{obj}$, and then returning to the original route. The ego-vehicle is then simulated along this collision path at speed $v_{crash}$, guaranteeing a spatiotemporal overlap with the obstacle.
\end{enumerate}

The distribution of different perturbation types is presented in Table~\ref{tab:dream}.

\begin{table}[h]
\centering
\caption{\textbf{Composition of Heuristic Perturbation Types in the Augmented Dataset.} 
    We present the statistical distribution of the three augmentation strategies applied to expert trajectories to generate the 10.9 million synthetic samples. The perturbations include \textit{Velocity Modification} (increase / reduce speed), \textit{Forced Lane Deviation} (unjustified lane changes), and \textit{Forced Collision} (intersecting with obstacles), designed to diversify failure patterns for robust training.}
\vspace{-6pt}
\scalebox{1.0}{
% \begin{tabular}{>{\kern-\tabcolsep}*{8}{>$c<$}<{\kern-\tabcolsep}}
\begin{tabular}{@{}ccccc@{}}
\toprule
\multirow{2}{*}{\textbf{Percentage}} & Increase Speed & Reduce Speed & Forced Lane Change & Forced Collision  \cr 
\cmidrule(lr){2-5}
& 36.88\% & 37.18\% & 9.46\% & 16.48\%  \\
\bottomrule
\end{tabular}
}
\label{tab:dream}
\end{table}

\subsubsection{Summary}
Collectively, these two data generation strategies yielded approximately \textbf{12.9 million} pairs of rough trajectories and their corresponding critic annotations. This extensive dataset provides the rich supervisory signal necessary for the model to master the complex tasks of risk identification, strategic action recommendation, and subsequent trajectory refinement.
\begin{table*}[t]
\centering
\caption{\textbf{Multi-Ability Results on Bench2Drive Benchmark.} C/L refers to camera/LiDAR, * denote using expert feature distillation. We conduct three independent trials with different seeds and report the mean and variance of the results.}
\label{tab:multi-ability}
\resizebox{\textwidth}{!}{
% \begin{tabular}{>{\kern-\tabcolsep}*{8}{>$c<$}<{\kern-\tabcolsep}}
\begin{tabular}{@{}lcccccccc@{}}
\toprule
\multirow{2}{*}{\textbf{Method}} & \multirow{2}{*}{\textbf{Modality}} & \multirow{2}{*}{\textbf{Venue}} & \multicolumn{6}{c}{\textbf{Ability (\%) $\uparrow$}}  \cr
\cmidrule(lr){4-9}

& & & Merging & Overtaking & Emergency Brake & Give Way & Traffic Sign & Mean  \cr 
\midrule
TCP*~\cite{TCP} & C & {NeurIPS' 22} & 16.18 & 20.00 & 20.00 & 10.00 & 6.99 & 14.63  \cr
TCP-traj*~ & C & {NeurIPS' 22} & 8.89 & 24.29 & 51.67 & 40.00 & 46.28 & 34.22  \cr
UniAD-Base~\cite{hu2023uniad} & C & {CVPR' 23} & 14.10 & 17.78 & 21.67 & 10.00 & 14.21 & 15.55  \cr
ThinkTwice*~\cite{jia2023think} & C & {CVPR' 23} & 27.38 & 18.42 & 35.82 & 50.00 & 54.23 & 37.17  \cr
VAD~\cite{jiang2023vad} & C & {ICCV' 23} & 8.11 & 24.44 & 18.64 & 20.00 & 19.15 & 18.07  \cr
DriveAdaptor*~\cite{jia2023driveadapter} & C+L & {ICCV' 23} & 28.82 & 26.38 & 48.76 & 50.00 & 56.43 & 42.08  \cr
DriveTrans~\cite{jia2025drivetransformer} & C & {ICLR' 25} & 17.57 & 35.00 & 48.36 & 40.00 & 52.10 & 38.60  \cr
TransFuser++~\cite{zimmerlin2024hidden} & C+L & {arXiv' 24} & 58.75 & 57.77 & 83.33 & 40.00 & 82.11 & 64.39  \cr
ORION~\cite{fu2025orion} & C & {ICCV' 25} & 25.00 & 71.11 & 78.33 & 30.00 & 69.15 & 54.72  \cr
Simlingo~\cite{renz2025simlingo}  & C & {CVPR' 25} & 54.01 & 57.04 & \textbf{88.33} & \textbf{53.33} & \textbf{82.45} & 67.03  \cr
\midrule
\cellcolor[HTML]{F2F2F2} CriticVLA~(\textit{Ours}) & \cellcolor[HTML]{F2F2F2} C & \cellcolor[HTML]{F2F2F2} - & \cellcolor[HTML]{F2F2F2} \textbf{61.28\small{$\pm$1.30}} & \cellcolor[HTML]{F2F2F2} \textbf{76.30\small{$\pm$1.28}} & \cellcolor[HTML]{F2F2F2} \textbf{88.33\small{$\pm$0.00}} & \cellcolor[HTML]{F2F2F2} 50.00\small{$\pm$0.00} & \cellcolor[HTML]{F2F2F2} 81.06\small{$\pm$0.91} & \cellcolor[HTML]{F2F2F2} \textbf{71.39\small{$\pm$0.26}} \cr
\bottomrule
\end{tabular}
}
\end{table*}

\subsection{Examples of the CriticDrive Dataset}
Figure~\ref{fig:CriticDrive1},~\ref{fig:CriticDrive2},~\ref{fig:CriticDrive3},~\ref{fig:CriticDrive4}  present a comprehensive visualization of some data samples from our constructed CriticDrive dataset, illustrating the alignment between geometric trajectories and semantic reasoning. 

The visual component (left panel) depicts the driving scenario from both the egocentric front-view and the simplified Bird's-Eye View bounding box locations. To facilitate analysis, we project future speed waypoints onto these views using a distinct color scheme: \textcolor{green}{\textbf{green}} dots represent the expert ground truth trajectory, while \textcolor{cyan}{\textbf{cyan}} dots indicate the initial ``rough'' trajectory generated by Stage-1 model. The future positions of surrounding dynamic actors are marked in \textcolor{blue}{\textbf{blue}}. Crucially, potential safety hazards are visualized in \textcolor{red}{\textbf{red}}, marking specific points where the rough trajectory poses a collision risk with other traffic participants. Additionally, the bottom section illustrates the velocity profiles, comparing the speed evolution of the expert and rough trajectories over time.

The semantic component (right panel) showcases the generated reasoning chain derived from the visual data. It consists of three segments:
\begin{itemize}
    \item \textbf{Risk Analysis:} A categorical assessment of potential hazards (e.g., \texttt{Collision Risk:True},  \texttt{Speed Risk:False}).
    \item \textbf{Action Suggestions:} Natural language commands (e.g., ``\textit{maintain speed}'', ``\textit{proceed with caution and yield}'') synthesized to mitigate the identified risks.
    \item \textbf{Detail:} A log of the precise quantitative metrics---such as angular offsets, distance to collision, and velocity errors---that serve as the underlying evidence for the risk analysis and action generation.
\end{itemize}
This structured format ensures that the dataset provides not only the correct trajectory but also the interpretability required for safe planning.
\begin{figure}[h]
    \centering
    \includegraphics[width=\linewidth]{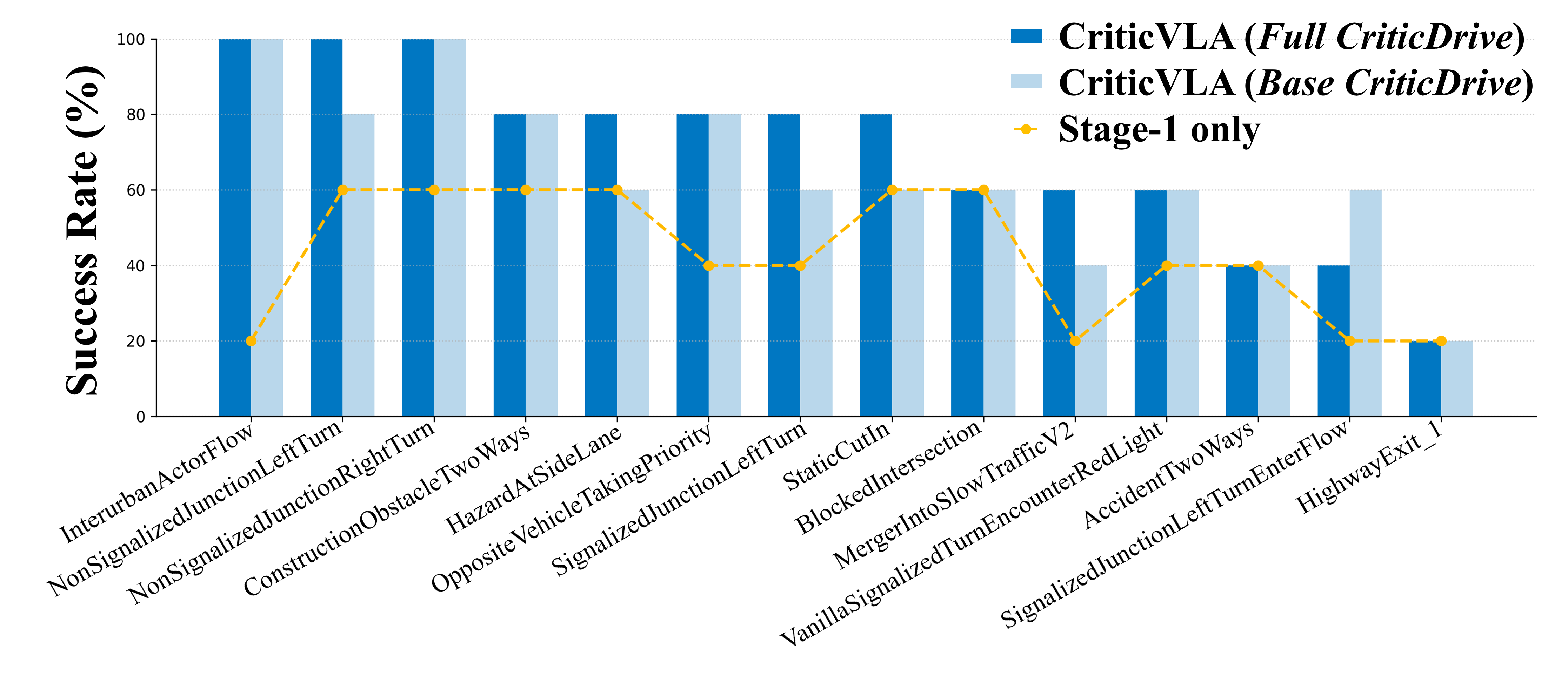} 
    \caption{Performance improvement of our method relative to the Stage-1 Model across different scenarios. The x-axis represents scenario categories. The y-axis denotes the per-scenario  success rate.}
    \label{fig:improvement_bar}
\end{figure}

\section{Core Set Details}
\label{sec:CoreSet}
% \subsection{Construction Method of the Core Set }
The standard Bench2Drive benchmark comprises 220 distinct routes, spanning 44 scenario categories with 5 variations per category. While comprehensive, we observed that a subset of these routes presents relatively trivial driving challenges, where baseline models frequently achieve near-perfect success rates. This performance saturation introduces statistical redundancy and, more critically, masks the nuanced performance gains derived from specific algorithmic improvements.

To address this and facilitate a more discriminative ablation analysis, we established the \textbf{Bench2Drive Core Set}.
\begin{itemize}
    \item \textbf{Selection Criteria:} From each of the 44 semantic scenario categories, we manually selected the single most representative and challenging route. This results in a compact subset of 44 routes that retains the semantic diversity of the full benchmark while filtering out trivial samples that contribute to noise in comparative evaluations.
    \item \textbf{Statistical Validation:} To ensure the Core Set remains a faithful proxy for the full distribution, we compared the performance of our CriticVLA on both sets. The model achieves a Success Rate (SR) of \textbf{72.73\%} on the Core Set, which aligns closely with the \textbf{73.33\%} SR observed on the Full Set (220 routes). This marginal deviation indicates that the Core Set successfully encapsulates the difficulty distribution of the original benchmark without introducing bias.
    \item \textbf{Utility for Ablation:} By eliminating non-discriminative easy samples, the Core Set amplifies the performance signals of different model variants. We utilize this set specifically for ablation studies—such as analyzing the sensitivity to initial rough trajectory quality or the impact of iterative refinement steps—where identifying subtle capability gaps is paramount.
\end{itemize}

Consequently, the Core-Set serves as a rigorous and efficient testbed, enabling us to more clearly discern performance differences between different model variants. Please refer to Table~\ref{tab:CoreSet} for the selected routes details.

\begin{table}[h]
\centering
\caption{\textbf{Core Set Details.} We list the IDs of all 5 routes corresponding to each scenario, as well as the route ID selected by the core set. The scenarios in the table are ordered alphabetically (A-Z). }
\vspace{-6pt}
\begin{tabular}{c|c|c}
\hline
Scenario & All routes & Core Set Selection  \\
\hline
AccidentTwoWays & 6, 36, 70, 172, 174 & 174 \\
Accident & 32, 33, 65, 148, 162 & 33 \\
BlockedIntersection & 20, 21, 49, 200, 204 & 200 \\
ConstructionObstacleTwoWays & 4, 5, 35, 127, 170 & 4 \\
ConstructionObstacle & 30, 31, 155, 160, 161 & 160 \\
ControlLoss & 80, 154, 159, 163, 190 & 80 \\
CrossingBicycleFlow & 54, 55, 56, 57, 58 & 57 \\
DynamicObjectCrossing & 122, 146, 147, 149, 153 & 146 \\
EnterActorFlow & 19, 93, 111, 112, 201 & 201 \\
HardBreakRoute & 79, 152, 158, 188, 192 & 158 \\
HazardAtSideLaneTwoWays & 71, 150, 173, 179, 180 & 180 \\
HazardAtSideLane & 2, 3, 164, 168, 171 & 2 \\
HighwayCutIn & 24, 51, 52, 53, 96 & 96 \\
HighwayExit & 128, 129, 130, 131, 140 & 131 \\
InterurbanActorFlow & 136, 137, 138, 139, 144 & 137 \\
InterurbanAdvancedActorFlow & 132, 133, 134, 141, 142 & 142 \\
InvadingTurn & 41, 42, 81, 82, 83 & 81 \\
MergerIntoSlowTrafficV2 & 135, 143, 184, 186, 189 & 135 \\
MergerIntoSlowTraffic & 22, 23, 50, 94, 95 & 22 \\
NonSignalizedJunctionLeftTurnEnterFlow & 208, 209, 214, 215, 216 & 216 \\
NonSignalizedJunctionLeftTurn & 10, 11, 12, 45, 198 & 10 \\
NonSignalizedJunctionRightTurn & 13, 46, 84, 85, 86 & 84 \\
OppositeVehicleRunningRedLight & 9, 43, 44, 194, 195 & 44 \\
OppositeVehicleTakingPriority & 14, 15, 16, 47, 87 & 47 \\
ParkedObstacleTwoWays & 37, 38, 72, 176, 177 & 177 \\
ParkedObstacle & 1, 34, 156, 165, 169 & 169 \\
ParkingCrossingPedestrian & 63, 64, 123, 145, 151 & 151 \\
ParkingCutIn & 0, 124, 125, 126, 157 & 157 \\
ParkingExit & 7, 77, 78, 183, 191 & 183\\
PedestrianCrossing & 113, 175, 202, 203, 205 & 205\\
SequentialLaneChange & 117, 118, 119, 120, 121 & 117\\
SignalizedJunctionLeftTurnEnterFlow & 210, 213, 217, 218, 219 & 219 \\
SignalizedJunctionLeftTurn & 103, 104, 105, 106, 107 & 107\\
SignalizedJunctionRightTurn & 8, 108, 196, 197, 199 & 199\\
StaticCutIn & 39, 40, 166, 185, 187 & 39 \\
T Junction & 193, 206, 207, 211, 212 & 206\\
VanillaNonSignalizedTurnEncounterStopsign & 29, 61, 62, 101, 102 & 29\\
VanillaNonSignalizedTurn & 26, 27, 28, 60, 100 & 60 \\
VanillaSignalizedTurnEncounterGreenLight & 114, 115, 116, 181, 182 & 181 \\
VanillaSignalizedTurnEncounterRedLight & 25, 59, 97, 98, 99 & 98 \\
VehicleOpensDoorTwoWays & 73, 74, 75, 76, 178 & 75 \\
VehicleTurningRoutePedestrian & 18, 91, 92, 109, 110 & 110 \\
VehicleTurningRoute & 17, 48, 88, 89, 90 & 88\\
YieldToEmergencyVehicle & 66, 67, 68, 69, 167 & 66 \\
\hline
\end{tabular}
\label{tab:CoreSet}
\end{table}

% \subsection{More Results with Core Set}

\section{Additional Results}
\label{sec:example}

\subsection{Bench2Drive Multi-Ability Results }
To assess CriticVLA's performance beyond comprehensive route completion, we employ a fine-grained \textbf{Multi-Ability Evaluation}. Based on the Bench2Drive benchmark, the Multi-Ability evaluation categorizes 220 routes into five distinct driving competencies based on the dominant scenario types encountered.

\vspace{0.2cm}
\noindent Multi-Ability evaluation measures 5 advanced driving capabilities:

\begin{itemize}
    \item \textbf{Overtaking:} This category evaluates the planner's ability to navigate around static or dynamic obstructions blocking the ego lane. It tests the agent's capability to perform safe lane borrowing and encroachment maneuvers while accounting for opposing traffic.
    \begin{itemize}
        \item \small{\textit{Core Competency}: Dynamic path planning and obstacle avoidance.}
        \item \small{\textit{Representative Scenarios}: Accident, ConstructionObstacle, ParkedObstacle, and HazardAtSideLane}.
    \end{itemize}

    \item \textbf{Merging:} This category assesses the agent's proficiency in integrating into dynamic traffic flows and navigating complex intersections. It requires the agent to estimate gaps in traffic and execute precise lateral control.
    \begin{itemize}
        \item \small{\textit{Core Competency:} Gap acceptance and lateral maneuverability in dense traffic.}
        \item \small{\textit{Representative Scenarios:} HighwayExit, InterurbanActorFlow, MergerIntoSlowTraffic, and various unprotected turning scenarios (e.g., NonSignalizedJunctionLeftTurn)}.
    \end{itemize}

    \item \textbf{Emergency Brake:} This category tests the agent's perception-reaction time and longitudinal stability when facing sudden, critical hazards. The agent must execute immediate braking to avoid collisions with vulnerable road users or erratic vehicles.
    \begin{itemize}
        \item \small{\textit{Core Competency:} Imminent hazard response and longitudinal emergency control.}
        \item \small{\textit{Representative Scenarios:} PedestrianCrossing, DynamicObjectCrossing, OppositeVehicleRunningRedLight, and ControlLoss.}
    \end{itemize}

    \item \textbf{Give Way:} This category focuses on the agent's adherence to right-of-way protocols where the ego vehicle is required to yield to other actors, prioritizing social driving behavior over trajectory efficiency.
    \begin{itemize}
        \item \small{\textit{Core Competency:} Priority prediction and yielding compliance.}
        \item \small{\textit{Representative Scenarios:} InvadingTurn (yielding to oncoming traffic encroaching on the lane) and YieldToEmergencyVehicle.}
    \end{itemize}

    \item \textbf{Traffic Signs:} This category evaluates the agent's perception logic and strict adherence to traffic regulations at intersections. It covers a wide range of junction topologies to ensure the agent correctly interprets regulatory signals.
    \begin{itemize}
        \item \small{\textit{Core Competency:} Regulatory perception and intersection handling logic.}
        \item \small{\textit{Representative Scenarios:} SignalizedJunction, NonSignalizedJunction, StopSign, and RedLight encounters.}
    \end{itemize}
\end{itemize}

\vspace{0.2cm}

We show the comparison of Multi-Ability results between our proposed CriticVLA and other baselines in Table~\ref{tab:multi-ability}. Comparing with SOTA baseline Simlingo~\cite{renz2025simlingo}, CriticVLA achieves +4.36\% improvements in the mean ability. Notably, the overtaking ability has increased by \textcolor{red}{+19.26\%} to reach 76.30\%, the merging ability has risen by \textcolor{red}{+7.27\%} to 61.28\%, and the emergency brake ability has achieved the same high score of 88.33\%. The overall improvement in driving capabilities demonstrates the superiority of the critic paradigm in enhancing driving performance. The significant improvement in overtaking aligns with the increased efficiency values in the Table 1 of the main text, indicating that our model does not compromise success rate by reducing driving speed.

The significant performance improvements compared to the baseline can be attributed to the fundamental advantages of CriticVLA paradigm (initial trajectory generation → linguistic risk analysis with improvement suggestions → refined trajectory). The initial trajectory and the risk analysis conducted on it enable the model to better understand the environment and its requirements for driving speed and direction. The subsequent refinement, guided by risk analysis, specifically resolves scenario-specific challenges: in merging, it optimizes lane-transition smoothness and conflict avoidance with surrounding vehicles; in overtaking, it enhances speed matching, safe distance maintenance, and timing judgments. This iterative process is particularly impactful for complex driving scenarios — where context-aware adjustments are critical. This conclusion is consistent with our analysis of the Table 1 in the main text.

\subsection{Performance Comparison by Scenario}
To comprehensively evaluate the effectiveness of our proposed CriticVLA, we conducted a comparative analysis across various driving scenarios. Figure~\ref{fig:improvement_bar} illustrates the performance gain of our CriticVLA compared to the Stage-1 Model. It can be observed that the proposed method significantly improves performance in many scenarios that previously exhibited poor driving outcomes, with some even achieving a 100\% success rate. Furthermore, the difference between training using \textit{Full CriticDrive} and \textit{Base CriticDrive} also indicate that EPAS subset data plays a crucial role in enhancing the model’s capabilities.

\subsubsection{Analysis on Improved Scenario}
Our method demonstrates significant improvements in scenarios requiring strict adherence to safety rules and longitudinal speed regulation. The core advantage stems from the \textbf{explicit risk verification mechanism}. By detecting four primary categories of risk (Collision, Speed, Direction, Context) and generating explicit corrective commands (e.g., \textit{``reduce speed"}, \textit{``adjust direction to the left"}), our method acts as a ``safety filter" that guides the refined trajectory generation.

\begin{itemize}
    \item \textbf{InterurbanActorFlow (+80\%):} In high-speed merging scenarios, the Base Model often fails to match the traffic flow speed. Our \textit{Speed Risk} module explicitly calculates the deviation between the predicted speed and the target flow speed. By generating a specific instruction like \textit{``increase speed from $v_{pred}$ to $v_{gt}$"}, the refinement module is conditioned to aggressively correct the longitudinal profile, ensuring a successful merge.
    
    \item \textbf{NonSignalizedJunctionRightTurn / LeftTurn (+40\% - 60\%):} Unprotected turns are prone to trajectory deviations. The Base Model might generate paths that cut corners or drift into opposing lanes. Our \textit{Direction Risk} analysis detects these angular and lateral offsets early. The resulting suggestion \textit{``adjust direction to the [Correct Side]"} provides a coarse but effective geometric constraint, forcing the refined trajectory back into the correct lane center.
    
    \item \textbf{OppositeVehicleTakingPriority (+20\%):} The improvement here is driven by the \textit{Collision Risk} check. While the Base Model might generate a path that spatially overlaps with an oncoming vehicle's future bounding box, our OBB (Oriented Bounding Box) intersection test flags this as a \texttt{collision\_risk}. The generated command \textit{``collision risk, proceed with caution and yield"} effectively suppresses the aggressive trajectory, forcing the agent to yield.
    
    \item \textbf{VanillaSignalizedTurnEncounterRedLight (+40\%):} The improvement highlights the effectiveness of the \textit{Traffic Light Risk} flag. The rule-based check identifies the red light status and issues a high-priority \textit{``Stop"} suggestion. This explicit semantic signal is more robust than the implicit visual features used by the Base Model, preventing red-light violations.
\end{itemize}

\subsubsection{Failure Case Analysis}

We observe performance degradation in specific scenarios. These failure cases reveal the limitations of the model and provide insights into potential future improvements.

\begin{itemize}
    
    \item \textbf{ConstructionObstacle:} Navigating construction zones requires precise, unstructured geometric maneuvering (e.g., ``nudge 0.5m left"). Our \textit{Direction Risk} module, which typically suggests \textit{``adjust direction to the left/right"}, is too coarse for this task. The generalized instruction may cause the refinement module to over-correct, leading to collisions with cones or barriers.
    
    \item \textbf{SequentialLaneChange:} This scenario requires a smooth, continuous multi-stage maneuver. Our method re-evaluates risks at every frame. If the risk analyzer oscillates between \textit{``maintain direction"} and \textit{``adjust direction"} across consecutive frames due to minor sensor noise, the resulting refined trajectory becomes jagged or unstable, leading to failure in completing the sequence.
    
    \item \textbf{YieldToEmergencyVehicle:} It is important to note that both our proposed method and the Base Model rely exclusively on front-view sensory inputs. Consequently, perceiving and yielding to emergency vehicles approaching from the rear presents an inherent challenge. We acknowledge this as a limitation of the current architecture; however, extending the framework to process multi-view inputs holds significant potential to resolve this issue in future work. Therefore, the 20\% success rate achieved by the Base Model is likely attributable to stochastic variance or chance, rather than genuine learned capability, and thus holds limited reference value.
\end{itemize}

\subsection{More Visualization Results}
As shown in Fig.~\ref{fig:more_vis}, the proposed CriticVLA provides effective corrections for both speed and steering, addressing issues such as directional drift, delayed braking, and insufficient deceleration during driving. In addition, the presented failure cases show that although CriticVLA actively performs corrective actions such as avoiding other vehicles, driving failures may still occur due to other factors.
% \subsection{The Code of the Proposed CriticVLA}

\textcolor{blue}{\textit{CriticVLA.mp4}} demonstrates examples where CriticVLA successfully performs direction correction and speed correction. Meanwhile, it presents a comparative analysis of the performance among CriticVLA, Simlingo, and the Stage-1 model. For further details, refer to the attached video.

\begin{figure}[h]
    \centering
    \includegraphics[width=\linewidth]{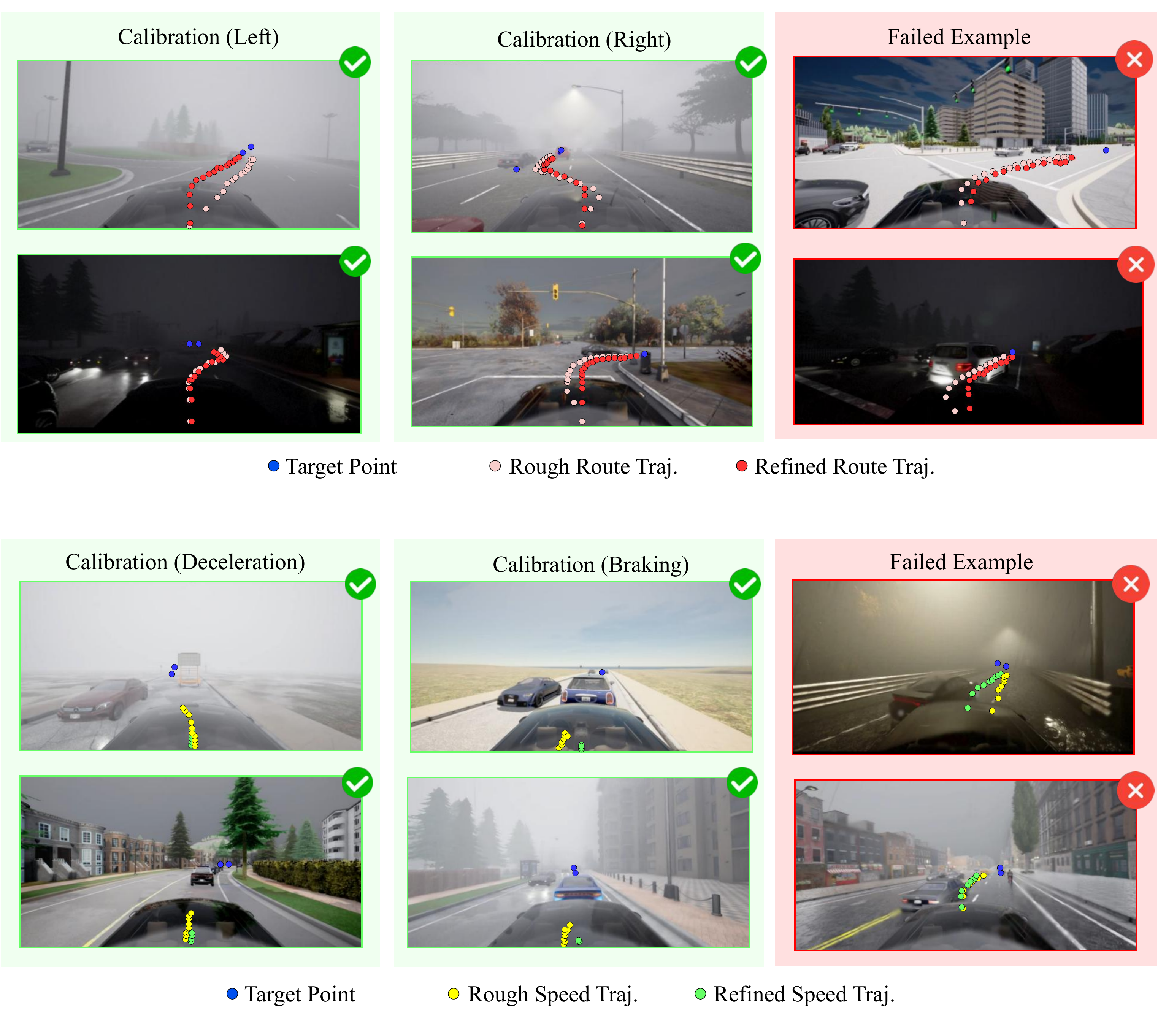} 
    \caption{More visualization results of proposed CriticVLA. The proposed CriticVLA exhibits excellent correction effects on both speed and direction, effectively rectifying issues such as direction deviation, delayed braking, or inadequate deceleration during vehicle operation. Additionally, the presented failed cases indicate that although CriticVLA intentionally performs corrective operations like avoiding other vehicles, there may still be other factors leading to driving failures.}
    \label{fig:more_vis}
\end{figure}

\begin{figure}%[t]
    \centering
    \includegraphics[width=0.99\linewidth]{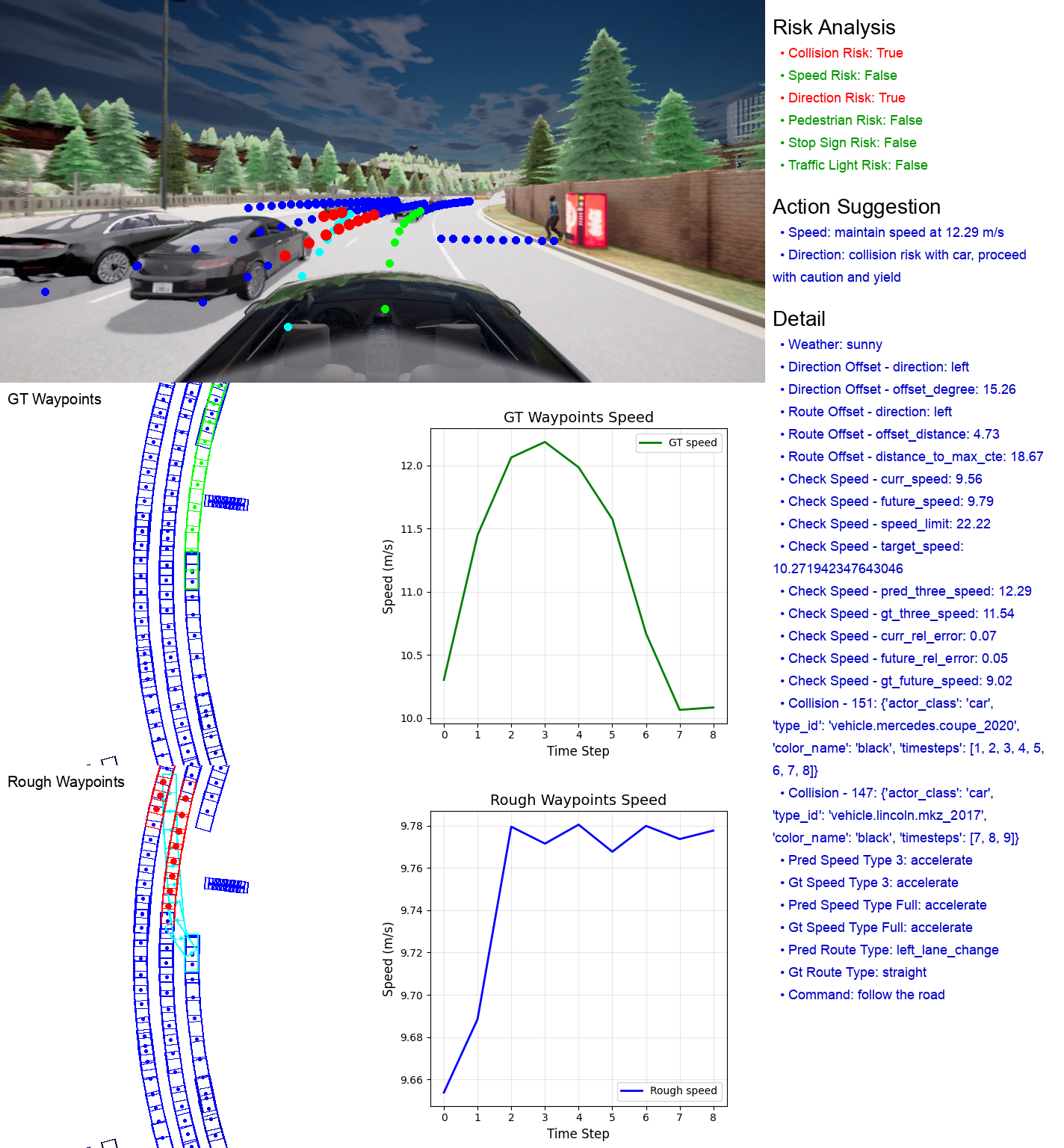}
    % \vspace{0.1cm}
    \caption{Examples of CriticDrive Dataset: collision and direction risk. \\ Left panel: driving scenario from both the egocentric front-view and the simplified Bird's-Eye View bounding box locations. Green dots represent the \textcolor{green}{\textbf{expert}} ground truth trajectory, while cyan dots indicate the \textcolor{cyan}{\textbf{rough}} trajectory. The future positions of surrounding \textcolor{blue}{\textbf{dynamic actors}} are marked in blue. Crucially, potential \textcolor{red}{\textbf{safety hazards}} are visualized in red, marking specific points where the rough trajectory poses a collision risk with other traffic participants. Additionally, the bottom section illustrates the velocity profiles, comparing the speed evolution of the expert and rough trajectories over time. \\
    Right panel: generated reasoning derived based on method in~\ref{sec:CriticDrive_construct}. Detected risks are marked in red, otherwise marked in green.}
    \label{fig:CriticDrive1}
\end{figure}

\begin{figure}%[t]
    \centering
    \includegraphics[width=0.99\linewidth]{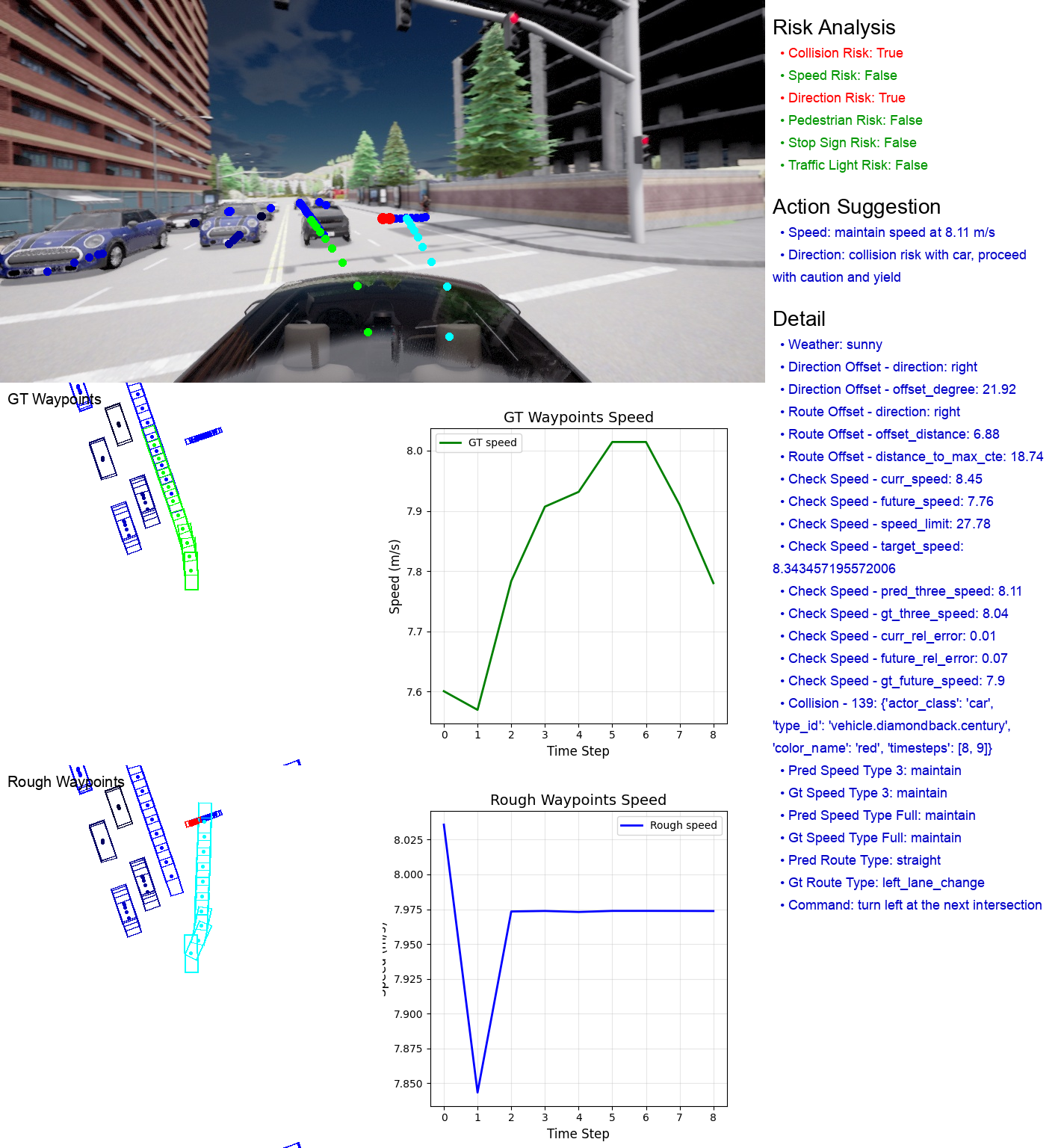}
    % \vspace{-0.3cm}
    \caption{Examples of CriticDrive Dataset: collision and direction risk. \\ Left panel: driving scenario from both the egocentric front-view and the simplified Bird's-Eye View bounding box locations. Green dots represent the \textcolor{green}{\textbf{expert}} ground truth trajectory, while cyan dots indicate the \textcolor{cyan}{\textbf{rough}} trajectory. The future positions of surrounding \textcolor{blue}{\textbf{dynamic actors}} are marked in blue. Crucially, potential \textcolor{red}{\textbf{safety hazards}} are visualized in red, marking specific points where the rough trajectory poses a collision risk with other traffic participants. Additionally, the bottom section illustrates the velocity profiles, comparing the speed evolution of the expert and rough trajectories over time. \\
    Right panel: generated reasoning derived based on method in~\ref{sec:CriticDrive_construct}. Detected risks are marked in red, otherwise marked in green.}
    \label{fig:CriticDrive2}
\end{figure}

\begin{figure}%[t]
    \centering
    \includegraphics[width=0.99\linewidth]{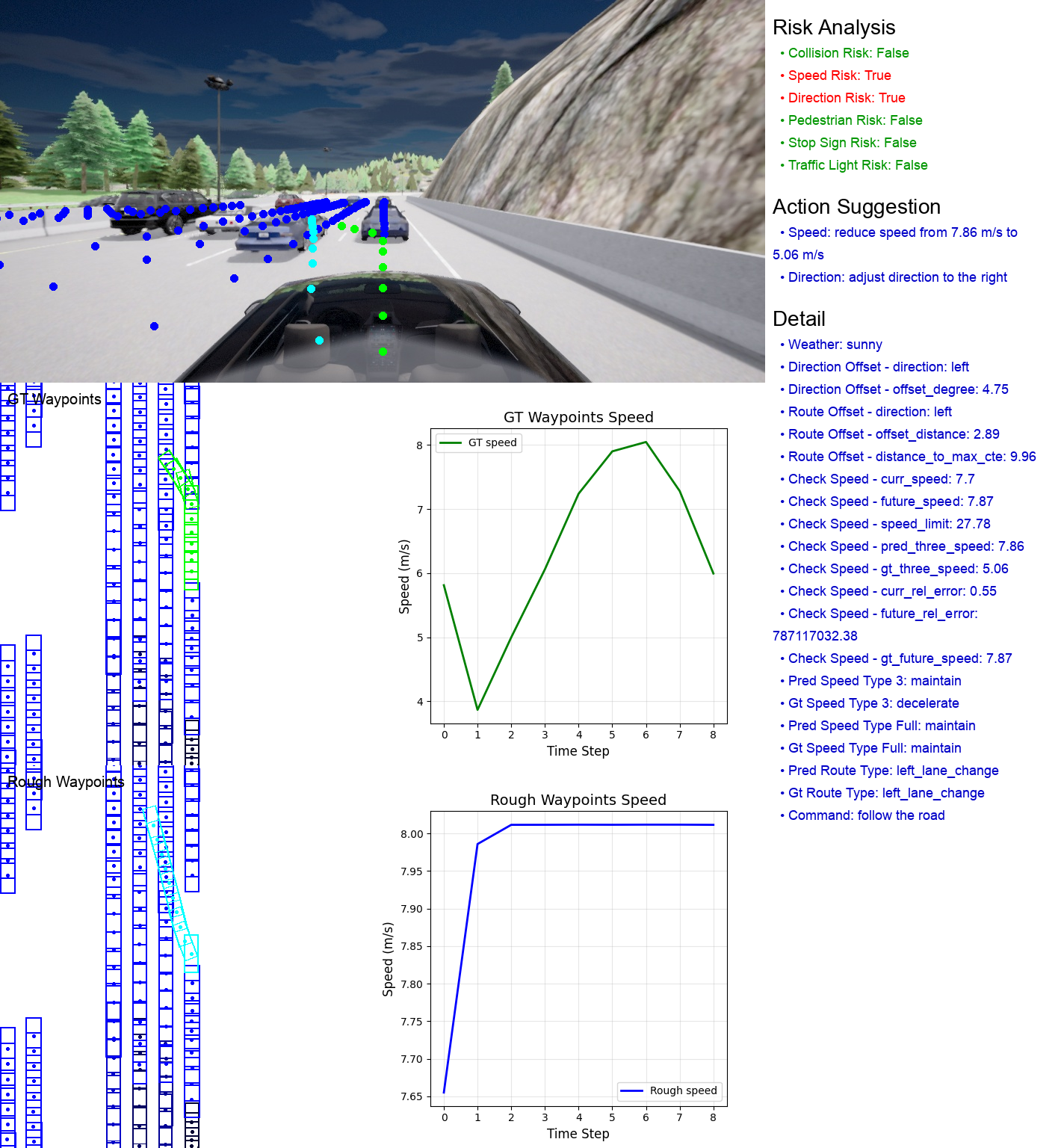}
    % \vspace{-0.3cm}
    \caption{Examples of CriticDrive Dataset: speed and direction risk. \\ Left panel: driving scenario from both the egocentric front-view and the simplified Bird's-Eye View bounding box locations. Green dots represent the \textcolor{green}{\textbf{expert}} ground truth trajectory, while cyan dots indicate the \textcolor{cyan}{\textbf{rough}} trajectory. The future positions of surrounding \textcolor{blue}{\textbf{dynamic actors}} are marked in blue. Crucially, potential \textcolor{red}{\textbf{safety hazards}} are visualized in red, marking specific points where the rough trajectory poses a collision risk with other traffic participants. Additionally, the bottom section illustrates the velocity profiles, comparing the speed evolution of the expert and rough trajectories over time. \\
    Right panel: generated reasoning derived based on method in~\ref{sec:CriticDrive_construct}. Detected risks are marked in red, otherwise marked in green.}
    \label{fig:CriticDrive3}
\end{figure}

\begin{figure}%[t]
    \centering
    \includegraphics[width=0.99\linewidth]{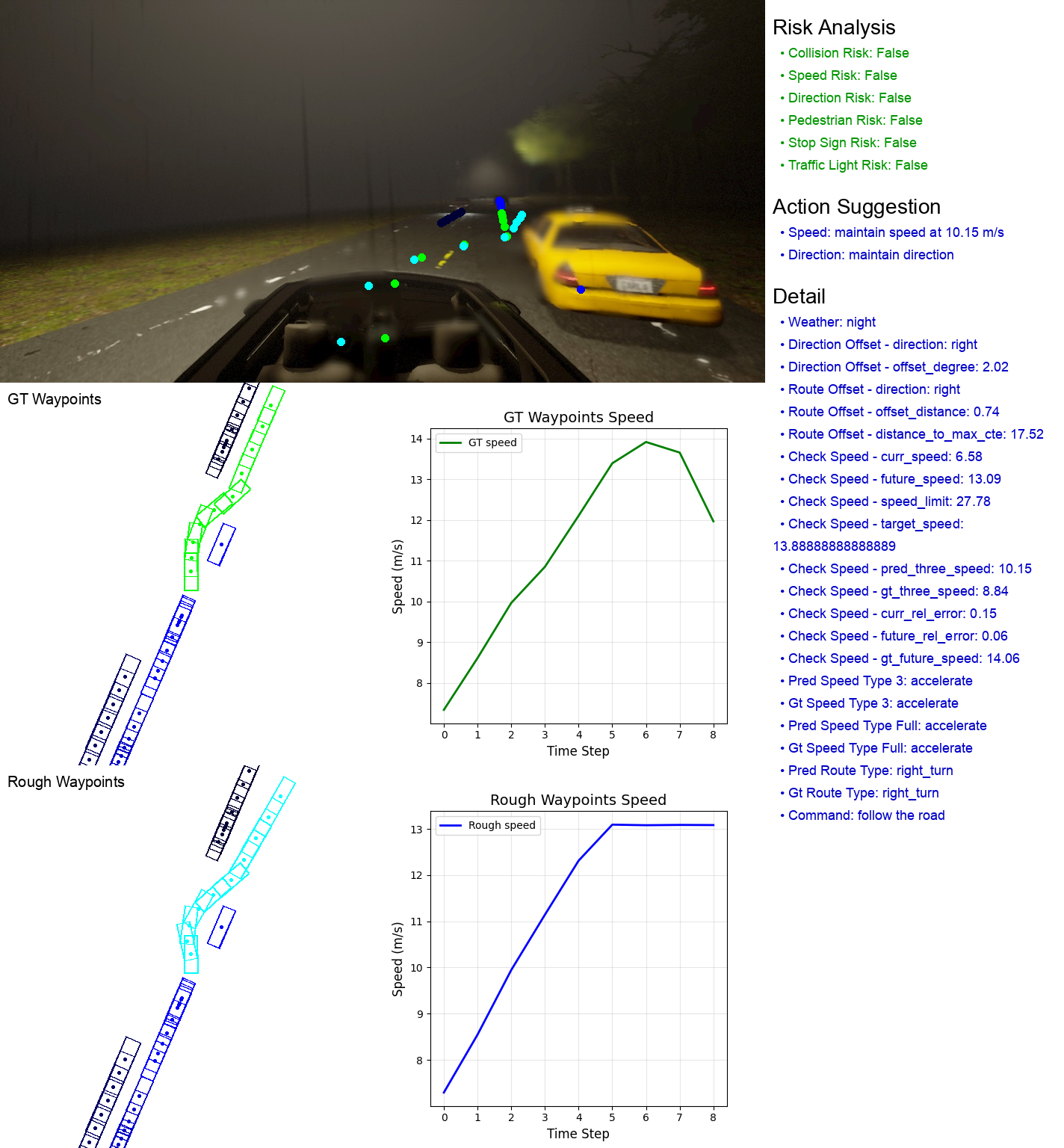}
    % \vspace{-0.3cm}
    \caption{Examples of CriticDrive Dataset: no risk detected. \\ Left panel: driving scenario from both the egocentric front-view and the simplified Bird's-Eye View bounding box locations. Green dots represent the \textcolor{green}{\textbf{expert}} ground truth trajectory, while cyan dots indicate the \textcolor{cyan}{\textbf{rough}} trajectory. The future positions of surrounding \textcolor{blue}{\textbf{dynamic actors}} are marked in blue. Crucially, potential \textcolor{red}{\textbf{safety hazards}} are visualized in red, marking specific points where the rough trajectory poses a collision risk with other traffic participants. Additionally, the bottom section illustrates the velocity profiles, comparing the speed evolution of the expert and rough trajectories over time. \\
    Right panel: generated reasoning derived based on method in~\ref{sec:CriticDrive_construct}. Detected risks are marked in red, otherwise marked in green.}
    \label{fig:CriticDrive4}
\end{figure}

\end{document}